\documentclass[]{style/iromlab}

\DeclareDocumentEnvironment{example}{}{\noindent\textbf{Running example:}\itshape}{}

\usepackage{ifthen}
\newboolean{include-notes}
\newboolean{include-new}
\newboolean{include-remove}
\setboolean{include-notes}{true}
\setboolean{include-new}{false}
\setboolean{include-remove}{false}

\usepackage[dvipsnames]{xcolor}
\usepackage[normalem]{ulem}
\newcommand{\justin}[1]{\ifthenelse{\boolean{include-notes}}{\textcolor{orange}{\textbf{Jaime:} #1}}{}}

\newcommand{\princeton}[1]{\ifthenelse{\boolean{include-notes}}{\textcolor{orange}{#1}}{}}

\usepackage{multicol}
\usepackage{caption}
\usepackage{amsmath, amsfonts, amssymb, amsthm}
\usepackage{enumerate}
\usepackage[inline]{enumitem}
\usepackage{mathtools}
\usepackage{algorithm}
\usepackage{subcaption}
\usepackage{graphicx}
\usepackage{longtable,tabularx}
\usepackage{placeins} 
\usepackage{float}
\usepackage{multirow}
\usepackage{bbm}
\usepackage{threeparttable}
\usepackage{balance}
\usepackage[ruled,algo2e]{algorithm2e}
\usepackage{adjustbox}
\usepackage{booktabs}
\usepackage{bm}
\usepackage{etoolbox}
\usepackage{microtype}
\usepackage{subcaption}
\usepackage[title]{appendix}
\usepackage{units}
\usepackage{cleveref}
\usepackage{xspace}
\usepackage{tcolorbox}
\usepackage{caption}
\usepackage{wrapfig}

\usepackage{wrapfig}
\usepackage{caption}
\usepackage{algorithmicx}
\usepackage{algpseudocode}

\DeclareCaptionType{algx}[Algorithm][List of Algorithms]

\makeatletter
\algrenewcommand\alglinenumber[1]{\footnotesize #1}
\algrenewcommand\algorithmicindent{1.2em}
\makeatother

\newcommand{\Denv}{\mathcal{D}_{\text{env}}}

\newcommand{\iidsim}{\overset{\text{iid}}{\sim}}
\newcommand{\Prob}{\mathbb{P}}

\newcommand{\simX}[1]{\tilde{X}_{#1}}
\newcommand{\simY}[1]{f(\tilde{X}_{#1})}

\newcommand{\method}{\textbf{SureSim}}
\newcommand{\methodClassicalWSR}{\textbf{Classical}\xspace}
\newcommand{\methodPPIUnif}{\method\xspace}
\newcommand{\methodControlVariate}{\textbf{Control Variate}\xspace}
\newcommand{\methodUBUnif}{\method\xspace-\textbf{UB}\xspace}

\newcommand{\methodPPINonAsym}{\method \xspace\textbf{(2-Stage)}\xspace}
\newcommand{\methodUB}{\method-\textbf{UB (2-Stage)}\xspace}

\newtheorem{theorem}{Theorem}

\newbool{extended}
\setbool{extended}{false}

\makeatletter
\newcommand{\longdash}[1][2em]{%
  \makebox[#1]{$\m@th\smash-\mkern-7mu\cleaders\hbox{$\mkern-2mu\smash-\mkern-2mu$}\hfill\mkern-7mu\smash-$}}
\makeatother
\newcommand{\omitskip}{\kern-\arraycolsep}

\author[1]{Apurva Badithela}
\author[1*]{David Snyder}
\author[1*]{Lihan Zha}
\author[2]{Joseph Mikhail}
\author[3\dagger]{Matthew O'Kelly}
\author[4\dagger]{Anushri Dixit}
\author[1]{Anirudha Majumdar}

\affiliation[1]{Princeton University}
\affiliation[2]{University of Texas, Austin}
\affiliation[3]{Waymo}
\affiliation[4]{University of California, Los Angeles}

\contribution[*]{Equal contribution}
\contribution[\dagger]{Equal advising.}
\begin{document}

\title{Reliable and Scalable Robot Policy Evaluation with Imperfect Simulators}


\abstract{
Rapid progress in imitation learning, foundation models, and large-scale datasets has led to robot manipulation policies that generalize to a wide-range of tasks and environments. However, rigorous evaluation of these policies remains a challenge. Typically in practice, robot policies are often evaluated on a small number of hardware trials without any statistical assurances. We present SureSim, a framework to augment large-scale simulation with relatively small-scale real-world testing to provide reliable inferences on the real-world performance of a policy. Our key idea is to formalize the problem of combining real and simulation evaluations as a prediction-powered inference problem, in which a small number of paired real and simulation evaluations are used to rectify bias in large-scale simulation. We then leverage non-asymptotic mean estimation algorithms to provide confidence intervals on mean policy performance. Using physics-based simulation, we evaluate both diffusion policy and multi-task fine-tuned \(\pi_0\) on a joint distribution of objects and initial conditions, and find that our approach saves over \(20-25\%\) of hardware evaluation effort to achieve similar bounds on policy performance. 
}


\website{
https://suresim-robot-eval.github.io  
}
{
suresim-robot-eval.github.io   
}



\maketitle


\begin{figure}[h]
    \centering
    \includegraphics[width=\linewidth]{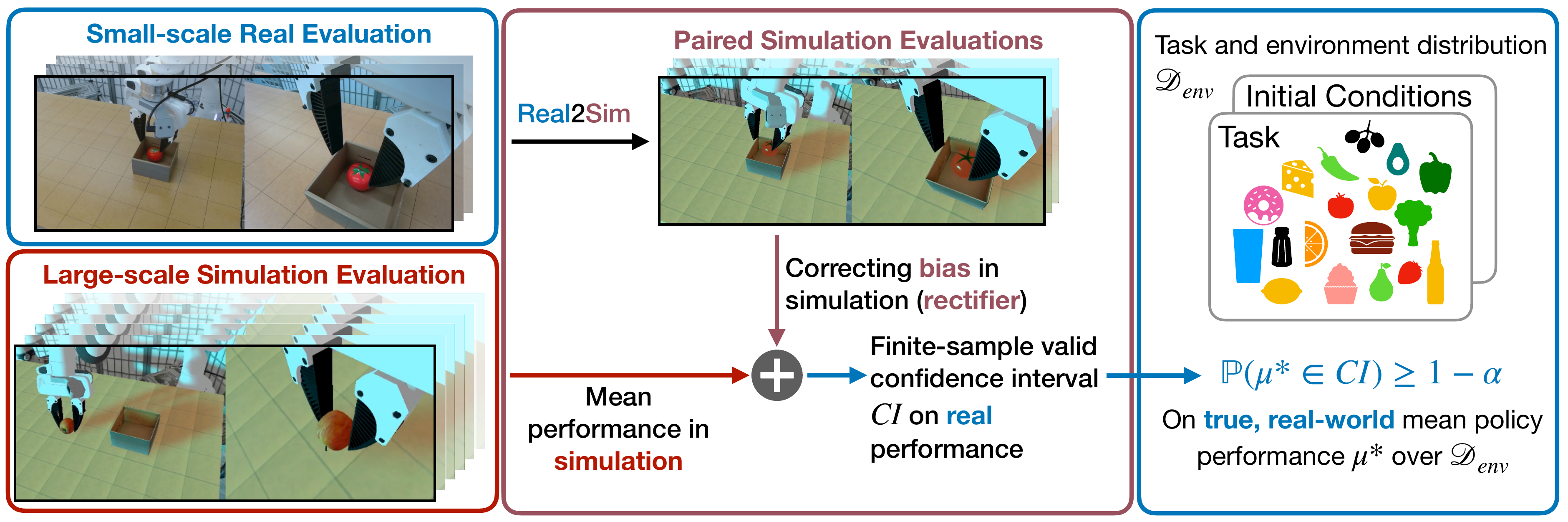}
    \caption{\small Our goal is to evaluate a policy by computing bounds on its mean real-world performance on a diverse environment distribution \(\mathcal{D}_{\text{env}}\). We present a framework that augments real-world evaluations with simulation evaluations to provide stronger inferences on real-world policy performance that could otherwise only be obtained by scaling up real-world evaluations.}
    \label{fig:overview}
\end{figure}
\section{Introduction}
\label{sec:intro}

Advancing robot learning requires statistically rigorous policy evaluation for reliably assessing how policies generalize to new tasks and environments~\cite{kress2024robot,gao2025taxonomy,barreiros2025careful}. Rapid progress in deep learning was driven by standardized metrics and evaluation benchmarks such as ImageNet~\cite{deng2009imagenet} and COCO~\cite{lin2014microsoft} in computer vision, and SquaD~\cite{rajpurkar2016squad} and GLUE/SuperGLUE~\cite{wang2018glue,wang2019superglue} in natural language. Unlike the static benchmarks in these domains, robot policy evaluation in the real-world requires physical interaction of the robot and its environment which is resource intensive in time and human effort. For instance, consider the fundamental question of evaluating the success rate of a policy on a distribution of environments. Due to the expensive nature of real-world evaluation, most research studies report empirical success rates of policies evaluated on a small number (e.g., 20-40) of trials. At the same time, there is a growing consensus for rigorous statistical analysis and nuanced discussions of evaluation criteria and policy failure modes~\cite{snyder2025your,kress2024robot,barreiros2025careful}. As a result, assessing whether a policy will perform reliably in a new environment distribution remains a core challenge~\cite{gao2025taxonomy}. 

In robotic manipulation, recent advances in physics-based simulators~\cite{taomaniskill3} and action-conditioned video prediction models~\cite{jang2025dreamgen,quevedo2025evaluating} provide scalable alternatives to real-world policy evaluation. While growing evidence suggests that simulation performance correlates well with real-world performance in aggregate across a diversity of tasks and environments~\cite{li2024evaluating,1x_world}, the simulation-to-real gap precludes rigorous statistical inferences about real-world outcomes from simulation results alone. This paper presents a framework to augment a small number of real-world evaluations with large-scale simulations to achieve scalable policy evaluation with trustworthy statistical inferences about real-world performance. Crucially, our framework can achieve tighter statistical bounds on policy performance by scaling up the number of simulations in place of scaling the number of real-world evaluations.

However, using large-scale simulations for policy evaluation with trustworthy statistical inferences on real performance faces significant challenges due to the simulation-to-real gap, stemming from mismatches in visual features (e.g., lighting conditions, object textures) and inaccurate modeling of contact physics and real physical parameters (e.g., friction coefficients)~\cite{zhou2025autoeval,pfaff2025scalable}. Current robot policies, including foundation models and imitation learning-based policies, can be sensitive to such discrepancies. As a result, performance bounds solely relying on large-scale simulation predictions can be biased. 


We tackle the aforementioned challenges in order to provide confidence intervals on mean performance of robot manipulation policies. Our key idea is to connect the problem of combining simulated and real-world evaluations to that of prediction powered inference (PPI)~\cite{angelopoulos2023prediction,angelopoulos2023ppi++}. PPI is a paradigm for valid statistical inference that can leverage a large set of learned model predictions together with a comparatively small number of gold-standard data.
In our setting, gold-standard labels are real-world evaluations of a policy, while the predictions are obtained from simulation evaluations. For a bounded performance metric, the resulting confidence intervals on mean policy performance are valid with the finite samples of real and simulated data. When simulation is sufficiently predictive, PPI yields tighter non-asymptotic confidence bounds than using real-world trials alone, allowing us to scale with simulation rather than costly hardware evaluations. Our experiments show that it is possible to save \(20-25\%\) of hardware trials using state-of-the-art physics-based simulators~\cite{taomaniskill3}.

\textbf{Statement of Contributions.} First, we present a rigorous policy evaluation framework for finite-sample valid inferences on real-world performance by combining large-scale simulation trials with a relatively small number of real trials. A key element is pairing each real trial with its corresponding simulation trial on the same task or environment instance to estimate and correct for simulation bias. To operationalize this, we introduce a real2sim pipeline to leverage prediction powered inference, and we identify best practices for integrating simulation with real-world evaluation to obtain tighter confidence intervals. Second, we demonstrate our evaluation paradigm on a single-task diffusion policy~\cite{chi2023diffusion} trained from scratch as well as the robot foundation model \(\pi_0\)~\cite{black2024pi_0} finetuned on multiple tasks. 
Finally, we discuss the sensitivity of our method to different types of real-simulation gap, including an example of when the  gap is too large for simulation to provide benefits over real-only trials. 

\section{Related Work}
\label{sec:related}
\textbf{Real-world Policy Evaluation.} 
Real-world evaluation is expensive because it offers limited parallelism while often requiring manual logging of outcomes and careful resetting of environments. Yet, it remains the gold-standard for assessing policy performance, driving significant efforts to establish standardized robotic benchmarks with carefully defined tasks, environments, and robot setups for reproducible policy evaluation~\cite{heo2023furniturebench,luo2025fmb,yang2019replab,khargonkar2024scenereplica,collins2023ramp}. 
A comprehensive list of best practices for rigorously evaluating robot policies is detailed in~\cite{kress2024robot}. 
To address these challenges, the community is building cloud-based evaluation platforms~\cite{pickem2017robotarium,zhou2023train,liu2021ocrtoc,bauer2022real,zhou2025autoeval} and distributed evaluator networks for unbiased, pairwise policy comparisons~\cite{atreya2025roboarena}. For example, AutoEval~\cite{zhou2025autoeval} autonomously classifies outcomes and resets environments using fine-tuned robot foundation models, while RoboArena~\cite{atreya2025roboarena} evaluates policies on the DROID platform~\cite{khazatsky2024droid}.
However, rapidly evaluating policies by scaling hardware evaluations with sufficient coverage for statistical assurances remains challenging.

\textbf{Policy Evaluation in Simulation.} Physics-based simulation benchmarks~\cite{todorov2012mujoco,tassa2018deepmind,makoviychuk2021isaac,taomaniskill3,zhu2020robosuite,nasiriany2024robocasa,rlbench,pumacay2024colosseum,zheng2022vlmbench,mees2022calvin} offer a reproducible and cheaper alternative to real-world robot policy evaluation. For example, SIMPLER~\cite{li2024evaluating} uses system identification and real2sim image editing methods to mitigate visual and dynamics discrepancies, showing that simulation-based performance rankings of vision-language-action models from match real-world rankings. However, setting up physics-based simulation can be time-consuming, especially when optimizing for visual fidelity and accurate matching of real-world dynamics. 
On the other hand, action-conditioned video world models~\cite{jang2025dreamgen,yang2023learning,brooks2024video} promise faster scene initialization via text, image, or video prompts~\cite{agarwal2025cosmos}. While their use in policy evaluation is nascent, it is attracting growing interest due to advantages over physics-based simulation~\cite{1x_world,quevedo2025evaluating}. These models, however, remain susceptible to hallucinations, and accurately capturing real-world dynamics is still a major challenge. Yet, simulation remains a valuable proxy. For example, simulation-based rankings—whether across different policies or for a given policy under diverse environmental factors—have been shown to correlate well with real-world performance~\cite{kadian2020sim2real,li2024evaluating,1x_world,pumacay2024colosseum,heo2023furniturebench}. 
Predictive red-teaming algorithms~\cite{majumdar2025predictive} offer an alternative to simulation by predicting whether a policy will succeed in a new environment without policy rollouts, showing strong correlation between real and predicted performance rankings across various environmental factors. In contrast to these methods, our approach leverages simulation, even when imperfect, to provide assurances on \emph{real} policy performance over a distribution of tasks and environments. 

\textbf{Statistically Confident Policy Evaluation.} 
In end-to-end self-driving applications, scalable simulation-based evaluation using importance sampling was used to provide statistical confidence on the safety of a self-driving policy~\cite{o2018scalable}. However, real-world evaluation remains gold-standard since it is difficult to model the real distribution of environments in simulation, which can lead to a bias in the resulting guarantees. In manipulation, limited by real-world evaluation costs and the large diversity of environments to evaluate in, researchers typically compare policy performance using only 20-30 real-world trials. However, such small sample sizes are insufficient to draw statistically significant conclusions in policy comparisons~\cite{snyder2025your}. Recognizing this need for reliable policy evaluation, a recent study compares generalist large behavior models to single-task policy counterpart using rigorous statistical evaluation methods, incorporating A/B real-world testing, and comprehensive real-world and simulation trials with robust statistical analysis~\cite{barreiros2025careful}. Sequential policy comparison frameworks further reduce real-world evaluation cost while maintaining statistical validity under anytime stopping~\cite{snyder2025your}. 
Beyond comparing policies, it is also important to assess the individual policy performance. For binary success criteria, Vincent \emph{et al.}~\cite{vincent2024generalizable} provide optimal confidence intervals from real evaluations. For non-binary metrics, confidence intervals may be obtained from real-world evaluation samples via concentration inequalities (e.g., Hoeffding~\cite{hoeffding1963probability}), though this would require a large real evaluation budget. Instead, we scale simulation while requiring a small number of real evaluations to ensure reliability.

Finally, concurrent to our work are efforts that apply statistical inference techniques to off-policy evaluation~\cite{mandyam2025perry}, and the use of control variates to combine simulation evaluation with real-world logged data for evaluation in self-driving applications~\cite{luo2025leveraging}. While conceptually related, our work differs in two key respects. First, we present finite-sample valid confidence bounds on real-world performance of manipulation policies that are tighter than existing baselines. Secondly, we demonstrate the idea of combining real-world and simulation evaluations on robotic manipulation, which faces a unique set of challenges --- robot policies can be sensitive to small perturbations in the environment, and the robot and environment state are more tightly interdependent. 

\section{Problem Statement}
\label{sec:problem}
Let \(\Denv\) denote a distribution over real-world environments \(\mathcal{X}\) in which we wish to evaluate a robot policy \(\pi \in \Pi\). In robotic manipulation, this distribution could be defined by the diversity of objects and tasks, environmental factors (e.g., lighting, background, table texture), and spatial variations in object and robot poses, among others. We assume a bounded evaluation metric \(M: \mathcal{X} \times \Pi \rightarrow [0,1]\), such as a success/failure metric or a continuous metric for partial task completion.

We consider the mean estimation problem in policy evaluation, where the goal is to estimate the policy’s average performance according to metric \(M\) over the environment distribution \(\Denv\). Formally, we define mean policy performance \(\mu^*\) as: $$ \mu^* = \mathbb{E}_{X \sim \Denv} [Y(X)],$$
where \(Y(X)=M(X,\pi)\) is the outcome of evaluating policy \(\pi\) in environment \(X\) under metric \(M\).
For sampled environments \(X_1, \ldots, X_{n} \iidsim \Denv\), the outcomes of real-world policy evaluation according to metric \(M\) are denoted as \(Y_1,\ldots,Y_n\), respectively. We define the empirical evaluation sample as \(S_n = \{(X_1, Y_1), \ldots , (X_n, Y_n)\}\). Using the empirical data \(S_n\) we seek a confidence interval \(CI = (l, u)\) that contains \(\mu^*\) with high probability. Confidence intervals provide  bounds on the true performance of a policy with high probability, and can be useful for decision-making and policy comparison. Mathematically stated, for any significance level \(\alpha \in (0,1)\) and any finite number \(n\) of real-world evaluations, we seek a confidence interval \(CI\) such that: 
\begin{equation}
\label{eq:guarantee}
\Prob_{S_n \sim \mathcal{D}^{n}_{\text{env}}} (\mu^* \in CI) \geq 1 - \alpha,
\end{equation}
where the probability measure is defined over the draw of the empirical evaluation sample \(S_n\). Any method that satisfies~\Cref{eq:guarantee} is Type-I error controlling at significance level \(\alpha\). While the interval \([0,1]\) trivially satisfies this guarantee, it provides little insight; therefore, we seek a tight confidence interval satisfying \Cref{eq:guarantee}. Importantly, we make no assumptions on the distribution of $Y_i$ beyond measurability and the boundedness induced by the metric $M$. We do not require \emph{a priori} knowledge of a distribution family, the existence of a density, or other structural assumptions. 

A nonasymptotic confidence interval for \(\mu^*\) can be derived directly from the finite number of gold-standard evaluations \(Y_1, \ldots, Y_n\) using standard non-asymptotic methods like Hoeffding~\cite{hoeffding1963probability} or Bernstein inequalities, or more recent state-of-the-art betting-based methods~\cite{waudby2024estimating}. 
Ideally, we want a tight interval concentrated around \(\mu^*\), but collecting a large number of real-world evaluations is costly. In contrast, simulation evaluations are relatively cheap and scalable. This motivates the central question of our work: \emph{Can we make valid inferences on the real performance of a policy by augmenting a small amount of real-world evaluations with a large number of simulation evaluations?}
\section{Approach: Simulation to Augment Real Tests via Prediction Powered Inference}
\label{sec:ppi}
 Suppose we have access to a simulator for policy evaluation. Let \(\mathcal{X}_{\text{sim}}\) denote the set of simulation environments. We can define a real2sim function \(g: \mathcal{X} \rightarrow \mathcal{X}_{\text{sim}}\) that translates a real environment setup into simulation. For each real environment \(X \in \mathcal{X}\), the corresponding simulation environment is defined as \(\tilde{X} = g(X)\). For example, as shown in~\Cref{fig:overview}, if \(X\) is a robot manipulation environment---defined by the robot (type, dynamics, texture, and initial pose), the objects (3D models, textures, material properties, and initial pose), and background conditions (lighting and background textures) --- then the corresponding simulation environment \(\tilde{X}\) is constructed to closely match the real-world dynamics and visual features. Simulation evaluations are given by the function \(f: \mathcal{X}_{\text{sim}} \rightarrow [0,1]\) defined as \(f(\tilde{X}) = M_{\text{sim}}(\tilde{X}, \pi)\), where \(\tilde{X} \in \mathcal{X}_{\text{sim}}\) and \(M_{\text{sim}}\) simulation evaluation metric. 
 
Correcting for bias in large-scale simulation predictions and deriving valid confidence intervals for \(\mu^*\) requires more than simply combining real and simulated evaluations through imputation. To tackle this challenge, we identify prediction powered inference (PPI)~\cite{angelopoulos2023prediction} as a suitable mathematical framework for our problem. Prediction powered inference enables valid statistical inference when experimental datasets are supplemented with machine-learning predictions. It has been applied to diverse problems such as protein structure analysis with AlphaFold, galaxy classification, and deforestation monitoring using computer vision~\cite{angelopoulos2023prediction}. For example, in galaxy classification, human annotators provide a limited set of ground-truth labels (“spiral” vs. “not spiral”) from galaxy images, while computer vision models provide cheaper predictions on the input images at a much larger-scale. 

\textbf{SureSim.} In our setting, each input corresponds to a robot manipulation environment \(X\), with the ground-truth label given by real outcome \(Y(X)\) of rolling out the policy. We choose simulation as a proxy for real-world evaluation but unlike the problems studied in~\cite{angelopoulos2023prediction}, we cannot directly evaluate on \(X\) in simulation. Composing the real2sim function with the simulator yields simulation predictions for real environments: \(f(\tilde{X}) = f(g(X))\). This formulation enables us to apply prediction-powered inference to rigorously combine real-world trials with large-scale simulation for reliable estimates of real performance.
 To apply PPI, we require a small number of paired evaluations in both real and simulation. For the set of \(n+N\) real environments \(X_1, \ldots X_{n+N} \iidsim \Denv\), we apply the real2sim function to get simulation environments \(\simX{1}, \ldots, \simX{n+N}\), where \(\simX{i} = g(X_i)\). The corresponding outcomes of evaluating the policy in simulation are denoted as \(f(\simX{1}),\ldots, f(\simX{n+N})\). Uniformly at random, we select \(n\) of those environments in which to conduct real trials. Thus, the paired evaluation data comprises of the real-world outcomes and associated simulation predictions: \(D_{\text{paired}} =  \{(Y_i, f(\simX{i}))\}_{i=1}^n\). The remaining number of additional simulation evaluations \(N\) exceeds the number of real-world evaluations \(n\), and these are denoted as \(D_{\text{sim}} = \{\simY{i}\}_{i=n+1}^{n+N}\). The \(i^{th}\) data sample is defined as:
\begin{equation}
    \Delta_i = \frac{n+N}{n} \big(Y_i - f(\tilde{X}_i)\big)\xi_i + f(\tilde{X}_i),\, 1 \leq i \leq n+N,
    \label{eq:unif_data}
\end{equation}
where \(\xi_i\) is an indicator of whether \((Y_i, f(\tilde{X}_i)) \in D_{\text{paired}}\).
The sample mean of~\Cref{eq:unif_data} yields the uniform PPI estimator for \(\mu^*\)~\cite{zrnic2024active}:
 \begin{equation}
    \mu^{\text{unif}}_{\text{PPI}} = \underbrace{\frac{1}{n}\sum_{i=N}^{n+N}(Y_i - f(\tilde{X}_i))}_{\text{Rectifier}} \, + \underbrace{\frac{1}{n+N}\sum_{i=1}^{n+N}f(\tilde{X}_i)}_{\text{Simulation evaluations}}.
    \label{eq:ppi_unif}
\end{equation}
The first term is referred as the rectifier, since it adjusts the bias in simulation predictions. For some significance level \(\alpha\), a confidence interval for \(\mu^{*}\) is computed from the sample evaluation data (\Cref{eq:unif_data}) using non-asymptotic methods for mean estimation via the Waudby-Smith and Ramdas (WSR) algorithm~\cite{waudby2024estimating}, which just requires the bounds of the random variable to be specified a priori. 
This method is denoted as \methodPPIUnif (\textbf{S}calable and \textbf{R}eliable Policy \textbf{E}valuation with \textbf{Sim}ulation), and the framework is summarized in~\Cref{alg:ppi_eval}. We also present a hedged variant termed \methodUBUnif which returns a confidence interval resulting from a union bound of \methodPPIUnif (computed at budget \(\frac{3 \alpha}{4}\)) and \methodClassicalWSR (at budget \(\frac{\alpha}{4}\)).

\methodPPINonAsym. Prediction-powered inference was originally introduced with a two-stage setup for data sampling~\cite{angelopoulos2023prediction}. This approach considers two sets of environments drawn i.i.d from \(\Denv\): the first set consists of a small number \(n\) of environments for which we collect both real and paired simulation evaluations, and the second set consists of a large number \(N\) of additional simulations. The PPI estimator for mean estimation is defined as:
\begin{equation}
    \mu_{\text{PPI}} = \underbrace{\frac{1}{n}\sum_{i=1}^n (Y_i - f(\tilde{X}_i))}_{\text{Rectifier}} \qquad + \underbrace{\frac{1}{N}\sum_{i=1}^N f(\tilde{X}_i)}_{\text{Additional simulation evaluations}},
\end{equation}
where \(\mu_{\text{PPI}}\) is also an unbiased estimate of \(\mu^*\). If we assume real and simulation scores to lie in the range \([0,1]\), the rectifier is bounded between \([-1,1]\). For a significance level \(\alpha\), a confidence interval on \(\mu^*\) is computed by separately deriving confidence intervals for the rectifier at some significance level \(\delta < \alpha\) and for the additional simulation data at significance \(\alpha-\delta\), and taking their Minkowski sum~\cite{angelopoulos2023prediction}.\footnote{The allocation of risk to $\delta$ and $\alpha-\delta$ can be approximately optimized. In practical settings, using $\delta \approx 0.9\alpha$ is a reliable heuristic.} For mean estimation, it can be proven that the true mean \(\mu^*\) lies in the resulting confidence interval with probability \(1-\alpha\)~\cite{angelopoulos2023prediction}.
To obtain finite sample guarantees, the rectifier and prediction confidence intervals are computed using WSR~\cite{waudby2024estimating}. In this two-stage approach, the bloating of the rectifier bounds coupled with the small number \(n\) of rectifier samples introduces inefficiencies in the resulting confidence interval. Similar to \methodUBUnif, we also introduce a  hedged version of this method termed \methodUB. 

\begin{theorem}
    \textbf{SureSim} and its variants return finite-sample valid confidence interval \(CI\) that satisfies~\Cref{eq:guarantee}.
\end{theorem}
\begin{proof}
    By construction of the real2sim function, the prediction rule is the functional composition \(f \circ g: \mathcal{X} \rightarrow [0,1]\). Under the assumption that \(\{X_i\}_{i=1}^{n+N}\) are drawn i.i.d from the task and distribution \(\Denv\), the finite-sample validity of the resulting confidence interval follows directly from~\cite{angelopoulos2023prediction,zrnic2024active}. 
\end{proof}

\begin{algorithm2e}[H]
\caption{\textbf{SureSim}}
\label{alg:ppi_eval}
\KwData{Real task and environment distribution \(\Denv\), Real-to-sim function \(g\), Policy \(\pi\), Real metric \(M\), Simulation metric \(M_{\text{sim}}\), Significance levels \(0 < \delta < \alpha < 1\)}
\KwResult{Confidence interval $CI$ on true mean \(\mu^*\)}
    Sample environments \(X_1 \ldots, X_{n+N} \sim \Denv\) \\
    \For{$i \gets 1$ \KwTo $n+N$}{
        \(\tilde{X}_i \gets g(X_i)\) \Comment{Apply real2sim function}\\ 
        \(f(\tilde{X}_i) \gets M_{\text{sim}}(\tilde{X}_i, \pi)\) \Comment{Simulation evaluation}
    }
    \For{$i \gets 1$ \KwTo $n$}{
        \(Y_i \gets M(X_i, \pi)\) \Comment{Real evaluation}
    }
    \(D_{\text{paired}} = \{(Y_i, f(\tilde{X}_i))\}_{i=1}^n\), \(D_{\text{sim}} = \{f(\tilde{X}_i)\}_{i=n+1}^{n+N}\) \\
     $CI \gets \textsc{PPI}(D_{\text{paired}}, D_{\text{sim}}, f, n, N, \alpha, \delta)$ \Comment{\textsc{UniformPPI} or \textsc{2-StagePPI}} \\
\Return{$CI$}
\end{algorithm2e}
\textbf{Baseline.} Termed as \textbf{Classical}, our primary baseline computes finite-sample confidence intervals --- without augmenting simulation --- by applying the non-asymptotic WSR procedure~\cite{angelopoulos2023prediction,waudby2024estimating} directly on real evaluations. These intervals represent the standard procedure for obtaining confidence intervals on the mean, and serve as an ablation with respect to the incorporation of proxy data. 

\textbf{Related Methods.} While we primarily compare our methods to \textbf{Classical} since it provides a finite-sample guarantee, we also implement and discuss the control variates procedure (denoted  \methodControlVariate)~\cite{luo2025leveraging} in the sim2sim setting. 
We do not consider it a baseline for the hardware experiments because it is not provably Type-I error controlling in finite samples. Therefore, the practitioner cannot know for their problem that the resulting confidence interval from~\cite{luo2025leveraging} contains the true mean at a specified level of confidence. In particular, this procedure utilizes the empirical correlation of the simulation evaluations to make optimization-based reductions to the mean estimation, at the expense of looser dependence on the confidence level $\alpha$. These paired samples yield a variance estimate for the control variate estimator, which is subsequently used in Chebyshev’s inequality to derive a confidence interval for the mean~\cite{luo2025leveraging}. However, in finite-sample settings, the variance estimate may be biased for small \(n\), and even unbiased constructions (such as through data splitting) need not upper-bound the \emph{true} variance, a requirement for Chebyshev’s inequality.

\section{Robot Experiments}
\label{sec:experiments}
We illustrate our method on pick-and-place tasks and evaluate policy generalization across diverse pick object types and initial conditions. Specifically, we seek to address the following questions: 
\begin{enumerate}
    \item How tight are our confidence intervals relative to the baselines? What benefit does this translate to in terms of real-world evaluation cost?
    \item How does the confidence interval width decrease as we scale the number of simulation evaluations?
    \item How well does our method perform under high and low real-simulation correlation?
\end{enumerate}

\begin{wrapfigure}{r}{0.5\textwidth}
\centering
\vspace{-1em}        \includegraphics[width=0.5\textwidth]{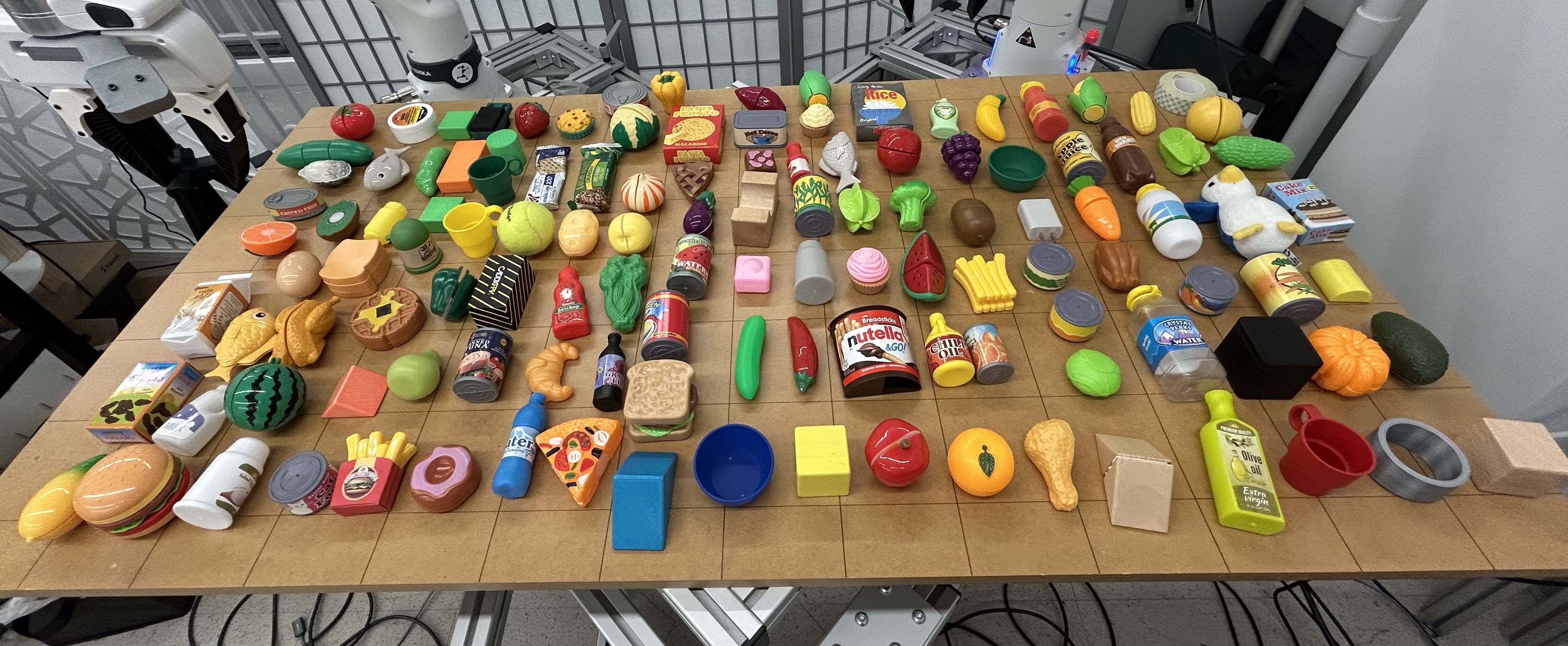}
\caption{Objects used for real-world evaluations.}
\vspace{-1em}  
\label{fig:real_objs}
\end{wrapfigure}
\textbf{Real2Sim Pipeline.} For evaluating object generalization, we gathered around 120 objects, most of which are toy kitchen items, shown in~\Cref{fig:real_objs}. To carry out paired evaluations in simulations, 3D models for these objects were obtained from a single image of the object using an off-the-shelf tool \href{https://www.meshy.ai/}{Meshy}. These 3D models were scaled to match real-world dimensions, and their pose was set according to real-world experiments. To construct the additional simulation dataset, we draw over 2100 objects from the RoboCASA repository~\cite{nasiriany2024robocasa}, which includes assets from Objaverse~\cite{deitke2023objaverse} and assets generated using a text-to-3D model ~\href{https://lumalabs.ai/}{Luma AI}. We filter out assets that are not semantically or geometrically equivalent—objects whose category or shape lacks a counterpart in the real-world set (e.g., plates).

In an ideal setting, we would have access to a large-scale repository of real-world objects paired with corresponding 3D models—akin to the YCB dataset~\cite{calli2017yale}, but expanded to include thousands of objects. This would allow us to uniformly sample a subset of objects for real-world and paired simulation evaluation, while using the remaining objects exclusively for additional simulation evaluations. However, these large datasets do not at present exist, and therefore, we take these \(120\) objects are taken to approximate the real-world distribution of objects that we wish to evaluate our policy on. 

\textbf{Experimental Setup.} We evaluate policies on a Franka Panda robot equipped with a wrist-mounted RealSense D405 and a Logitech C920 third-person camera. For simulation, we replicate the setup in ManiSkill3~\cite{taomaniskill3} by constructing a customized Franka Panda robot in which the default gripper is replaced with the 3D model used in our real-world experiments. The robot base pose in simulation is aligned with the real robot through manual calibration. Similarly, we transfer the camera calibration parameters from the real setup—covering both the wrist-mounted camera and the third-person camera—to their counterparts in simulation. We use the same control frequency as the robot in the real world.
The workspace table is constructed by scripting a table-like mesh and overlaying it with the texture of the real table. For the background, we import a real-world mesh obtained via 3D scanning. Finally, we use the default shader with shadows enabled to strike a balance between simulation speed and visual quality, and tune lighting parameters until policy performance in simulation on randomly selected initial conditions is as high as possible. 

\noindent
\textbf{Policies.} We evaluate two policies: i) a single-task diffusion policy~\cite{chi2023diffusion} trained from scratch and ii) a generalist policy \(\pi_0\)~\cite{black2024pi_0}, fine-tuned on multiple objects. Our diffusion policy is trained on 200 demonstrations of a single task --- to pick up a tomato and place it in a plate. The training distribution comprises of the tomato and the plate being placed randomly in a 30cm-by-40cm space. Though trained on a single object, we evaluate this diffusion policy on its generalization to multiple objects. We finetune \(\pi_0\) for 7 different objects according to the language instruction ``put \(<\)object\(>\) into the box" with 40 demonstrations for each object. In the fine-tuning demonstrations, the \emph{pick} object is randomly placed in a 10cm-by-20cm grid, while the box is placed at roughly the same xy-position. After each inference step, the open-loop action horizon was set to full action chunk size of 30.

\begin{wrapfigure}{r}{0.5\textwidth}
\centering
\vspace{-2mm}
    \begin{subfigure}[b]{0.2415\textwidth}  
        \centering
        \includegraphics[width=0.85\textwidth]{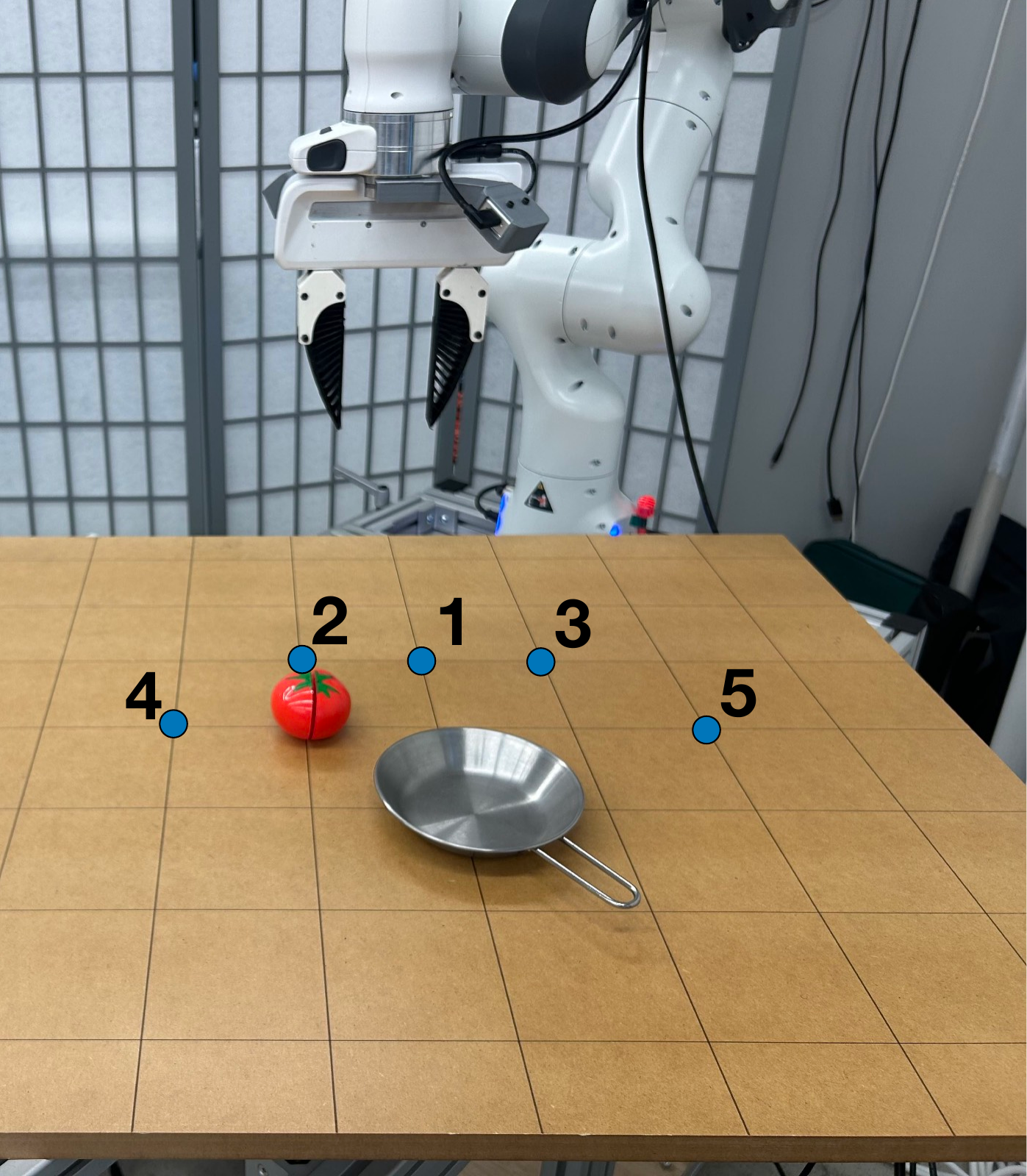}
        \caption{Diffusion Policy Setup}
        \label{fig:dp_eval_ic}
    \end{subfigure}
    \begin{subfigure}[b]{0.25\textwidth}
        \centering
        \includegraphics[width=0.84\textwidth]{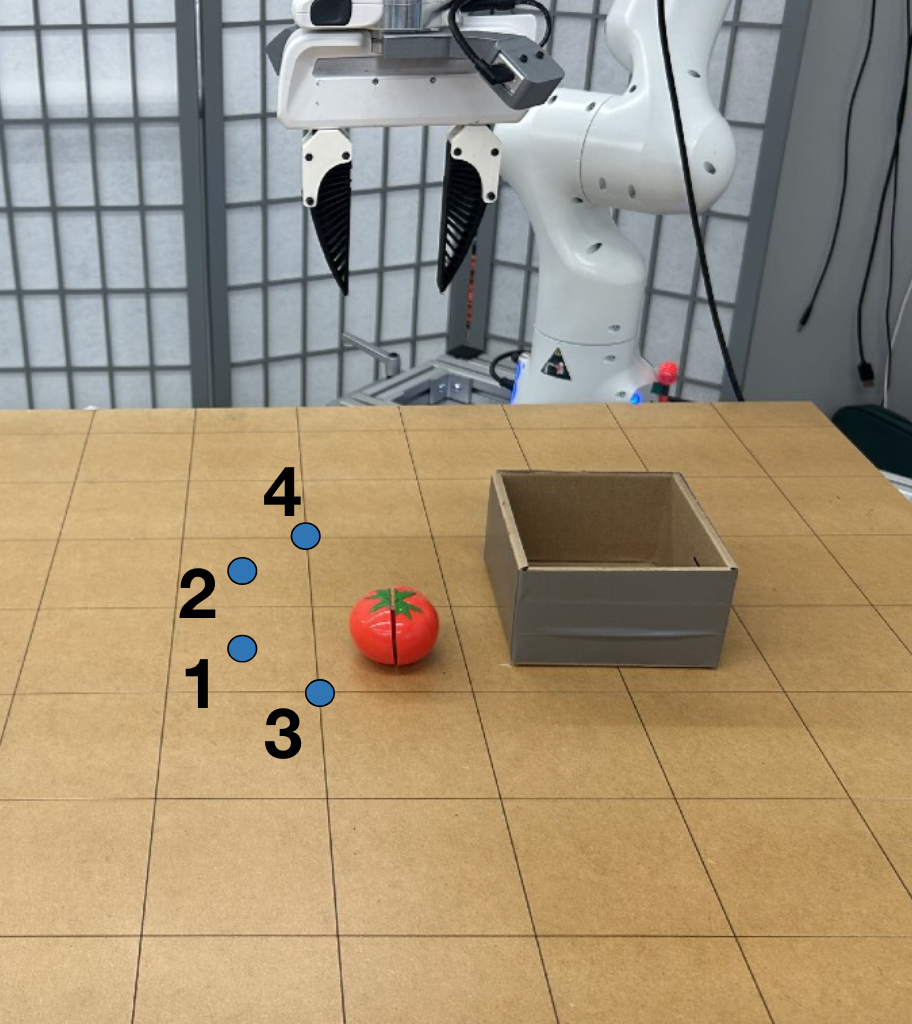}
        \caption{\(\pi_0\) Setup}
        \label{fig:pi0_eval_ic}
    \end{subfigure}
    \caption{Initial conditions for evaluation experiments}
    \vspace{-3mm}
\end{wrapfigure}

\noindent
\textbf{Evaluation Metrics.} For each real object, we rollout diffusion policy and \(\pi_0\) at five different initial conditions of the pick object as shown in~\Cref{fig:dp_eval_ic,fig:pi0_eval_ic}, respectively. Each rollout is assigned a partial evaluation score: 0 for no grasp, 0.25 for a failed grasp (object slips), 0.5 for a successful grasp, 0.75 for successful grasp but unsuccessful release over the place object, and 1 for complete task success. 
In simulation, we record a partial success score as follows: 0 for no grasp, 0.5 for successful grasp, and 1 for complete task success. 

\noindent
\textbf{Real-Simulation Evaluation Gap.} Additionally, we discuss the real-simulation evaluation gap for robot manipulation, and share a few insights to mitigate this. Crucially, for mean estimation, this gap manifests in the variance of the rectifier, which represents the difference in the real and simulation outcomes on the paired set \(D_{\text{paired}}\). That is, a high rectifier variance corresponds to low correlation on \(D_{\text{paired}}\) and a high real-simulation gap, while a low variance corresponds to high correlation and a small gap. Depending on the evaluation criteria used for constructing \(D_{\text{paired}}\), well-known sources of the real-simulation gap --- such as visual and dynamics discrepancies --- can reduce the reliability of simulation in predicting real outcomes and increase rectifier variance. For stochastic policies (e.g., diffusion policy which has randomness in the denoising process), this mismatch is further exacerbated by inconsistencies in policy seeding between real and simulated runs. For example, if we evaluate diffusion policy using a discrete evaluation metric over a set of initial conditions by pairing a single hardware trial with a simulation evaluation at the same initial condition, we are unlikely to see a high correlation on the paired set of evaluations. For the same real-simulation experimental setup, the rectifier variance can vary with the task and evaluation criteria, the policy under evaluation, and the axis of generalization considered in the distribution \(\Denv\). Together, these factors can lead to low correlation in paired evaluations, thereby diminishing the predictive utility of simulation and undermining the advantage of large-scale simulation for trustworthy inference on real performance. To address this issue, we implement the following measures: (1) we ensure that both real-world and simulation evaluations use the same random seed, and (2) in simulation, we sample 20 initial conditions from a 2cm-by-2cm box of the the real \((x,y)\) initial condition, execute the policy for each, and average the results to obtain a more robust estimate of the simulation counterpart. These measures are designed to mitigate the real-simulation gap without requiring additional real evaluations.

\subsection{Real2Sim Robot Experiments}
\label{sec:real2sim}

\begin{wrapfigure}{r}{0.6\textwidth} 
    \centering
    \vspace{-3mm}
        \includegraphics[width=\linewidth]{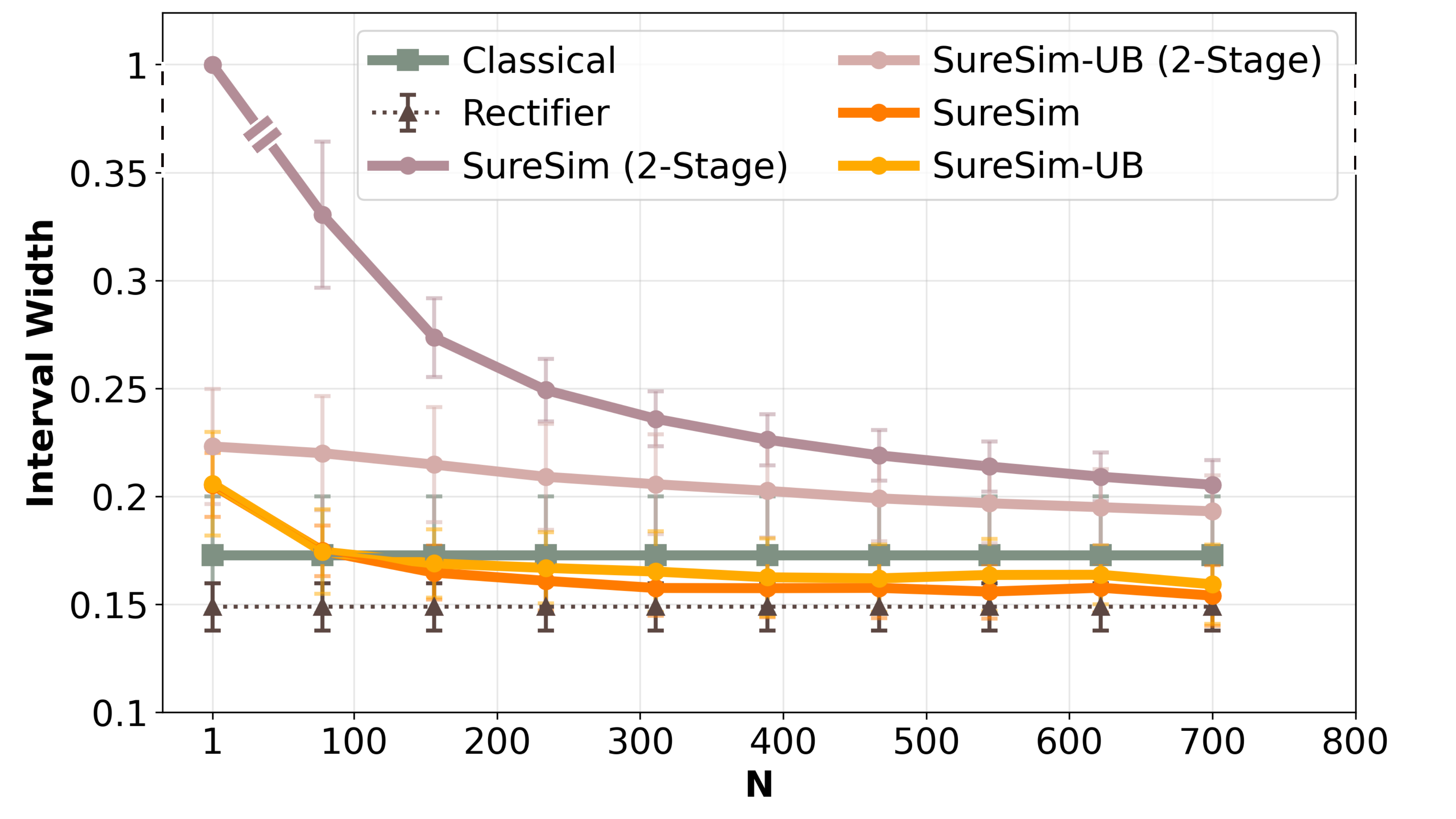}
        \caption{Evaluating Diffusion Policy with \(n=60\) paired trials and up to \(700\) additional simulations.}
        \vspace{-3mm}
        \label{fig:dp}  
\end{wrapfigure}
\textbf{Diffusion Policy.} First, we evaluate a single-task diffusion policy on a distribution of various types of pick objects. For each real object \(X_i\), we get the real label \(Y_i\) by taking the average of partial scores of trials conducted at \(5\) initial conditions shown in~\Cref{fig:dp_eval_ic}. The paired evaluation score \(f(\tilde{X}_i)\) is the average of simulation partial scores averaged over 100 simulation initial conditions corresponding to the 5 real-world initial conditions. On average, the correlation on the paired dataset is 0.72. For the additional simulation objects, we choose the Objaverse split of RoboCASA objects. For the following results, we use 100 random samples\footnote{In practice, this amounts to 100 random re-samplings of \(60\) objects (without replacement) from the bank of \(120\) real objects.} of \(n=60\) paired evaluations and up to \(N=700\) additional evaluations. For the \methodPPINonAsym family, the rectifier significance level is set to \(\delta=90\%\) of the the total significance level. All methods are given a significance level of \(\alpha=0.1\). 

\begin{wrapfigure}{r}{0.6\textwidth}    \centering
    \includegraphics[width=\linewidth]{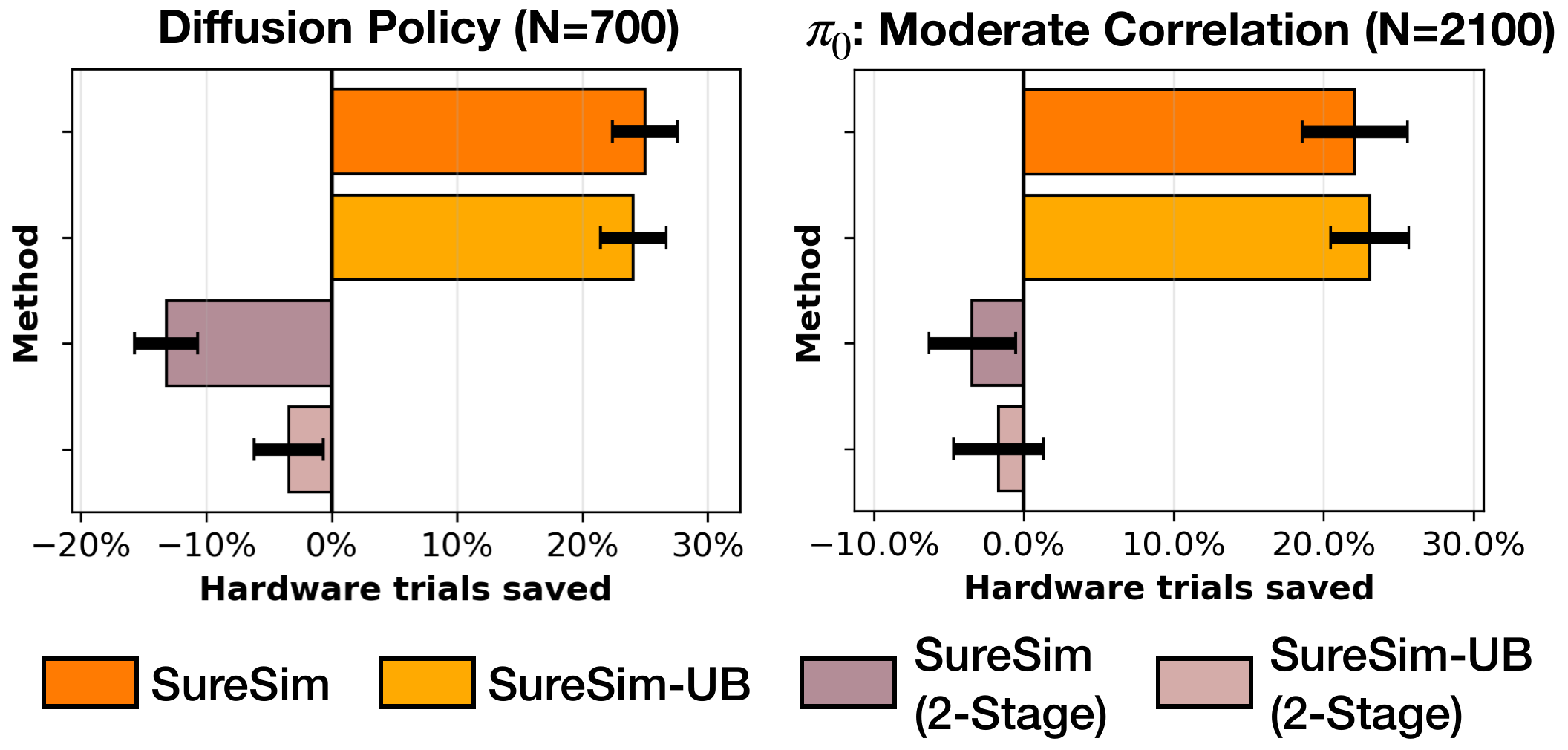}
\caption{Average number of hardware trials saved compared to \methodClassicalWSR, computed over \(100\) random draws of data. Error bars indicate standard error of the mean savings.}
\label{fig:savings}
\end{wrapfigure}
\Cref{fig:dp} illustrates the size of confidence interval widths as we scale-up simulation. At just \(100\) additional simulations, \methodPPIUnif tightens the confidence interval with respect to \methodClassicalWSR as simulation is scaled up further. In cases where the confidence interval is not truncated at \(0\) or \(1\), the rectifier interval width corresponds to a lower bound on the confidence interval width as the number of additional simulations increase. Here, the rectifier interval width is computed from finite-sample confidence intervals for the rectifier at \(\delta=0.09\) level of significance, and is determined by the rectifier variance. \methodPPIUnif and \methodUBUnif in~\Cref{fig:dp} approach this lower bound relatively quickly, indicating efficient usage in incorporating simulation data up to the limit imposed by the real-simulation gap. At \(N=700\), the mean interval width of \method\xspace is \(0.16\) which is a decrease of \(14.4\%\) compared to the interval width of length \(0.187\) for the \methodClassicalWSR method. The \methodPPINonAsym family has a slower decrease in interval width with scaling simulations as compared to \methodPPIUnif family due to: i) the two-stage procedure introducing inefficiencies in separately computing confidence intervals for the rectifier and additional simulations, and ii) the significance level allocated to the simulation confidence interval (\(\alpha -\delta = 0.01\)), requiring further simulation trials.

We study the advantage of our methods over hardware-only evaluations as follows. For each method, we compute a confidence interval at \(n=60\) samples, and iteratively search over the number of samples \(n\) given to \methodClassicalWSR until the resulting confidence interval is tighter than the method's interval.~\Cref{fig:savings} illustrates the resulting savings, where the \methodPPIUnif method family yields over \(25\%\) savings with respect to real-only evaluation. 

\begin{figure}[h]
    \centering
    \includegraphics[width=\linewidth]{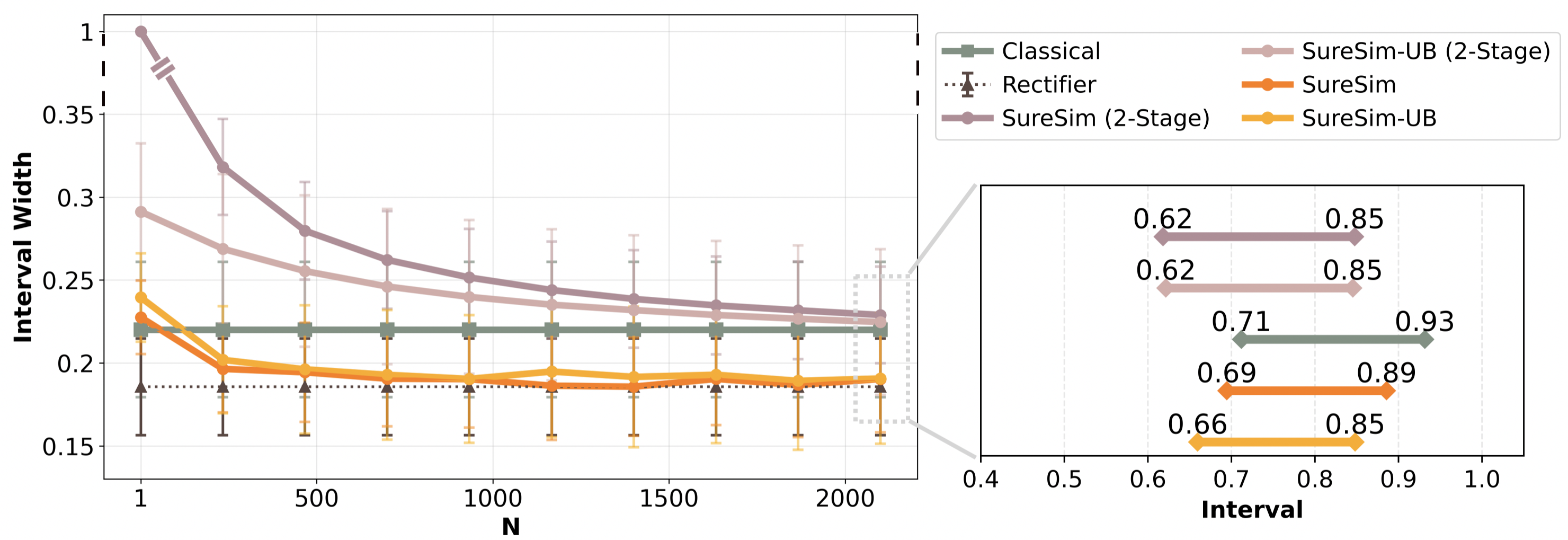}
    \caption{\textbf{How does interval width decrease with scaling simulations under moderate correlation?} This figure reports results for \(\pi_0\) evaluated on initial conditions \(\{1,2,3,4\}\) for \(n=60\) paired trials and over \(2100\) additional simulations.}
    \label{fig:pi0_real2sim}
\end{figure}
\textbf{Finetuned \(\pi_0\).} We present two examples of evaluating \(\pi_0\), where we consider a joint distribution over objects and initial conditions. In the first case, an object is randomly selected and placed at an initial condition sampled from \(\{1,2,3,4\}\), as shown in~\Cref{fig:pi0_eval_ic}. In the second case, the initial condition is sampled from \(\{1,2,3\}\). The former setting yields a moderate real-to-sim correlation, whereas the latter produces a low correlation. We present both cases to evaluate our methods under contrasting real-simulation correlation regimes. The real evaluation label for a specific object and initial condition is recorded according to the partial score metric, and the paired simulation records the average of simulation partial scores on the perturbed set of initial conditions corresponding to the real initial condition. 

\begin{wrapfigure}{r}{0.6\textwidth} 
    \centering
        \includegraphics[width=\linewidth]{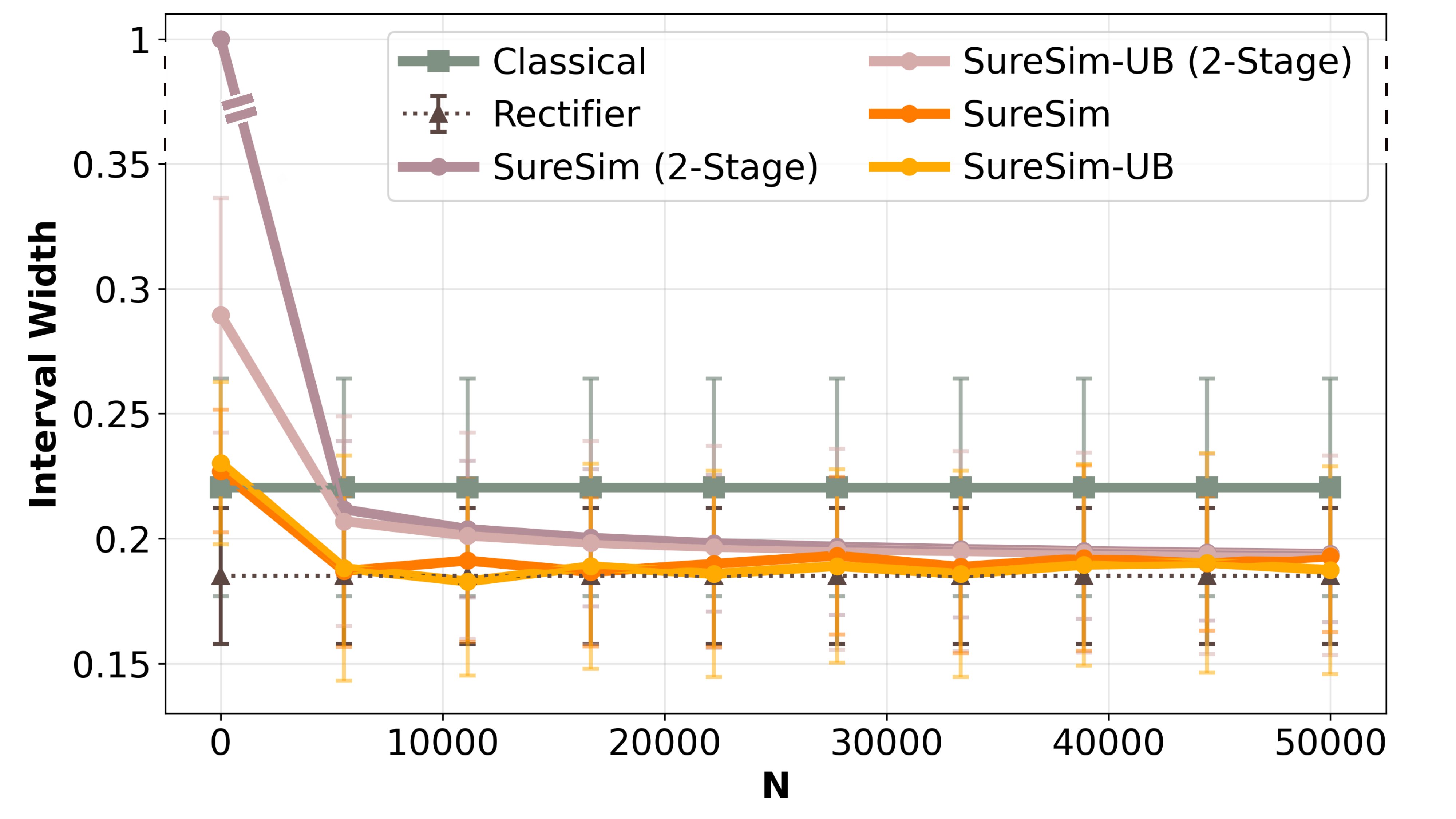}
        \caption{Evaluating \(\pi_0\) at \(n=60\) and scaling simulations up to \(N=50000\).}
        \label{fig:sanity_check}  
\end{wrapfigure}
\textit{Moderate Correlation.} In~\Cref{fig:pi0_real2sim}, we report average interval widths. Confidence intervals will vary in width and location for different draws of data from the same distribution; we illustrate one representative confidence interval in~\Cref{fig:pi0_real2sim}. The real-simulation Pearson correlation \(\rho\) is \(0.59\) on average on the paired evaluation set. The number of additional simulations is sufficient for the \method\xspace family to approach the rectifier lower bound and result in an advantage over \methodClassicalWSR. Further scaling up simulation would address the two-stage inefficiency in the \methodPPINonAsym family. However, the \method\xspace family converges by \(N=500\) additional simulations, indicating efficiency with scaling simulations compared to the \methodPPINonAsym family. As seen in~\Cref{fig:savings}, this leads to over a \(20\%\) decrease in real trials on average. In future work, we study further improvements to these finite-sample results by fine-tuning simulation.

Although our object repository is limited to approximately 2100 objects, we conduct a sanity check in~\Cref{fig:sanity_check} by sampling additional simulations with replacement, up to \(N=50,000\). While the rectifier interval width is limited by the difference in the real and simulation outcomes \((Y_i - f(\tilde{X}_i))\) on a small number of evaluations, additional simulation evaluations can be scaled up in the two-stage methods to reduce interval width. We observe that \methodPPINonAsym progressively approaches the rectifier lower bound as the number of simulations increases. \method\xspace and \methodUBUnif more efficiently converge to the rectifier lower bound within \(5000\) additional simulations, illustrating an upper bound on the benefit that additional simulation can provide given the real–simulation gap.

\textit{Low Correlation.} In~\Cref{fig:pi0_real2sim_zero_corr}, we present results for a low correlation case, where initial conditions are sampled from \(\{1,2,3\}\) shown in~\Cref{fig:pi0_eval_ic}. Empirically, the finetuned \(\pi_0\) demonstrates strong generalization to diverse object types despite being finetuned on only \(7\) objects. The initial conditions \(\{1,2,3\}\) achieve higher success rates across object types compared to initial condition \(4\). While these initial conditions are correspondingly easy and difficult in simulation, the predictive signal from simulation is insufficient to capture subtle variations in real-world performance, resulting in a low correlation of around \(-0.05\). 

This results in qualitatively different behavior. As seen in~\Cref{fig:pi0_real2sim_zero_corr}, none of our methods beat \methodClassicalWSR. This is unsurprising because the there is nothing to infer about real policy performance from simulation. In particular, as listed in~\Cref{tab:exp}, the sample variance on real data is smaller than the sample rectifier variance. As a result, scaling up simulation does not help in reducing the variance in our estimates of the true mean (\Cref{eq:guarantee,eq:ppi_unif}).

\begin{figure}[h]
    \centering
    \includegraphics[width=\linewidth]{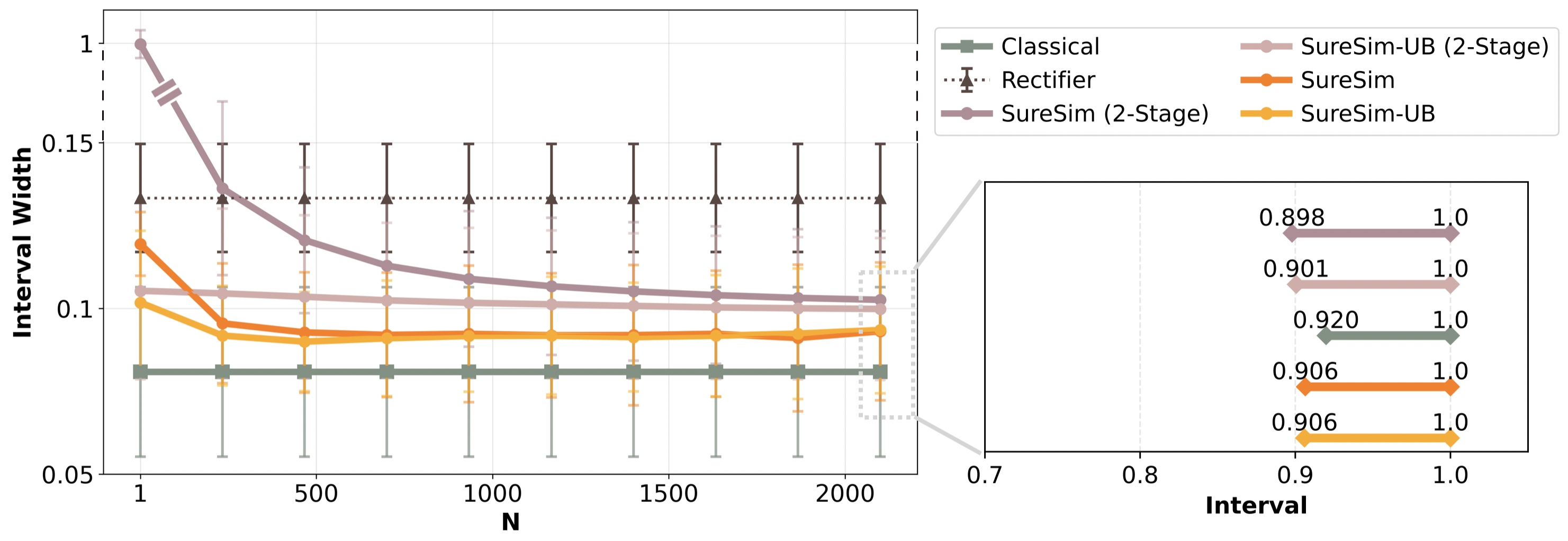}
\caption{\textbf{How does interval width decrease with scaling simulations under low correlation?} \(\pi_0\) evaluated on initial conditions \(\{1,2,3\}\) at \(n=60\). \emph{Left}: There is no decrease in interval width with scaling simulations, which is expected given the low correlation between paired real and simulation trials. Due to truncation, the interval widths are smaller than the rectifier interval width (which do not truncate here).}
\label{fig:pi0_real2sim_zero_corr}
\end{figure}

Across all Real2Sim experiments with moderate correlation, we observe that \method\xspace yields the greatest savings in terms of hardware trials saved and reduction in interval width. \method\xspace converges relatively quickly in the number of additional simulations in comparison to the two-stage methods. Furthermore, as we scale the number of simulations, the confidence intervals from our methods do not shrink to arbitrarily small widths. This controlled behavior is desirable, as it prevents overconfidence and ensures robust estimation of the mean. The gain from large-scale simulation depends on the correlation between paired real and simulated evaluations, which determines the rectifier variance. As shown in the asymptotic setting~\cite{angelopoulos2023prediction}, combining real and simulated data is effective only when the rectifier variance is smaller than the variance of real evaluations—a condition that remains necessary in the non-asymptotic regime before committing substantial effort to large-scale simulation.~\Cref{tab:exp} in~\Cref{sec:appendix_tab} summarizes key experimental parameters as well as the sample correlation, means, and variances for each of the experiments. 

\begin{figure}
    \centering
    \vspace{-1em}
    \includegraphics[width=0.95\textwidth]{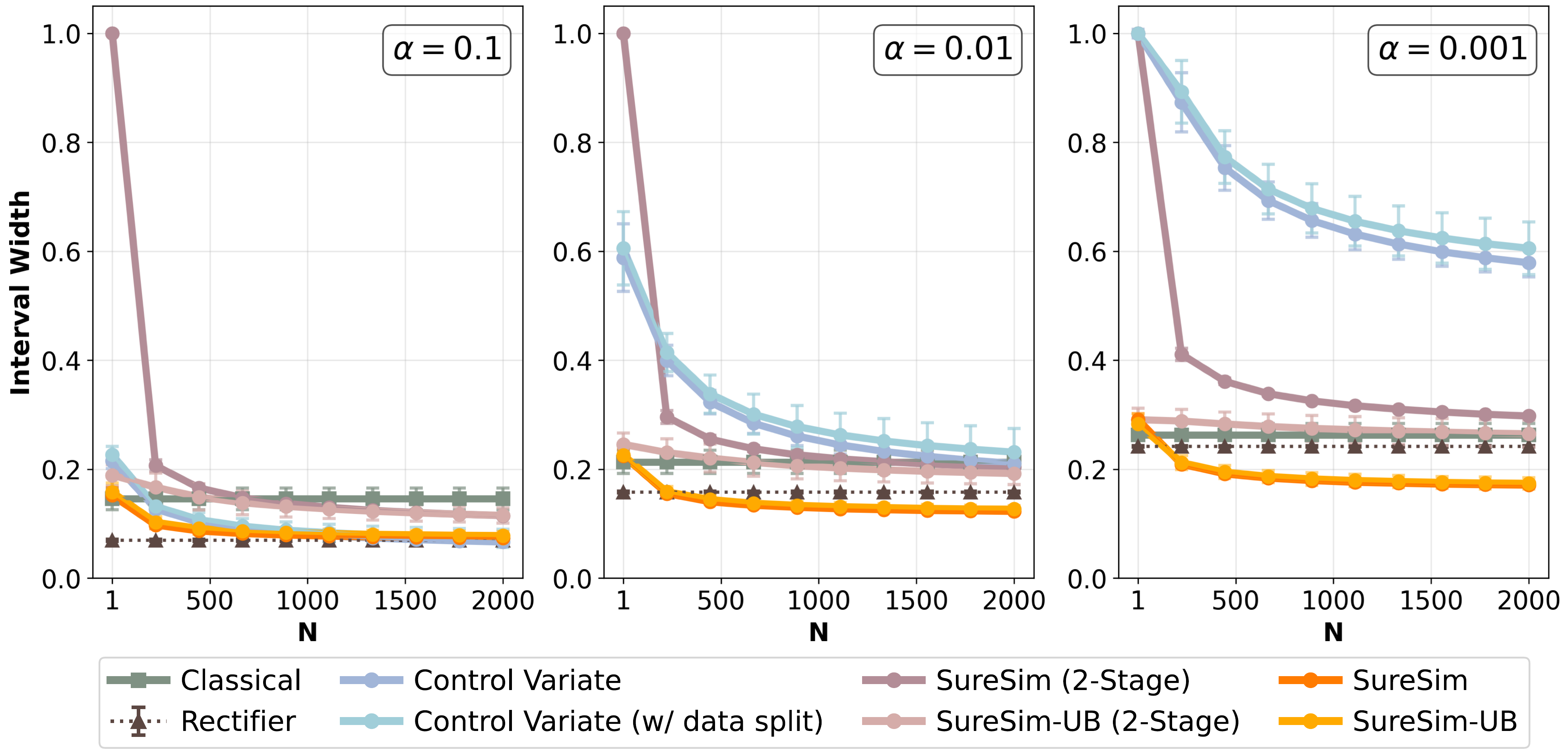}
    \caption{Prediction-powered inference for $n=100$ paired trials with up to 2000 additional simulations. Our methods always beat the classical baseline irrespective of increasing confidence levels.}
    \label{fig:sim2sim_pi0_all}
    \vspace{-1em}
\end{figure}
\subsection{Sim2Sim Experiments}
\begin{wrapfigure}{r}{0.45\textwidth}
\centering
\vspace{-1em}
    \begin{subfigure}[b]{0.49\linewidth}  
        \centering
\includegraphics[width=0.9\textwidth]{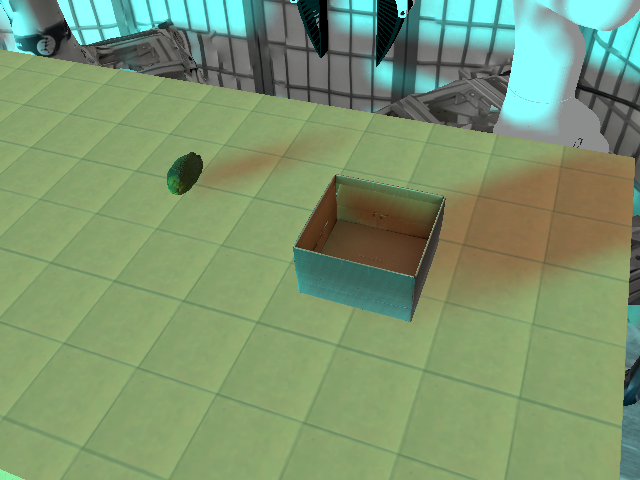}
        \caption{\([0,6,6]\)}
        \label{fig:real_light}
    \end{subfigure}
    \begin{subfigure}[b]{0.49\linewidth}
        \centering
        \includegraphics[width=0.9\textwidth]{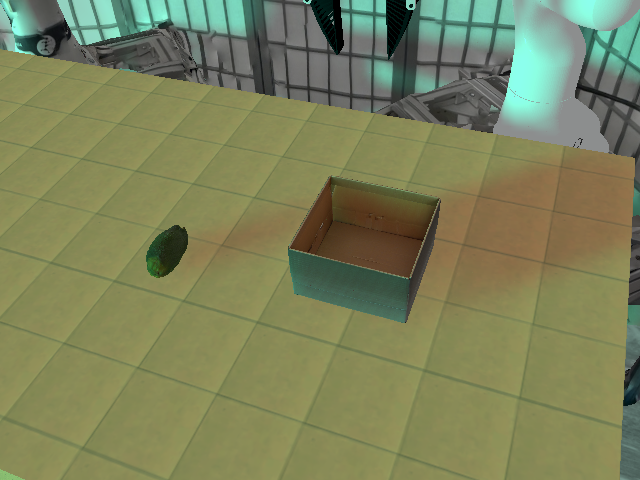}
        \caption{\([0,3,2]\)}
        \label{fig:sim_light}
    \end{subfigure}
    \caption{Sim2Sim setup with lighting parameters for ``real" (left) and ``sim" (right) settings.}
    \label{fig:sim2sim}
    \vspace{-1em}
\end{wrapfigure}
To illustrate coverage rate of confidence intervals, we run simulation-simulation experiments where we use a larger number of simulation evaluations as heldout samples to compute the ``true" mean and validate coverage. As illustrated in~\Cref{fig:sim2sim}, we use one simulator setting as the ``real" environment and the other as ``simulation". We evaluate finetuned \(\pi_0\) on 3D object models of real objects (\Cref{fig:real_objs}) as well RoboCASA. Once again, we consider a joint distribution over objects and initial conditions \(\{1,2,3,4\}\), and evaluate each trial according to the simulation partial score metric. The paired evaluation set has a very high correlation of around \(0.97\). We use 400 randomly drawn environments as heldout samples for validating coverage.

\begin{wrapfigure}{r}{0.54\textwidth}
\vspace{-2em}
\centering
    \includegraphics[width=\linewidth]{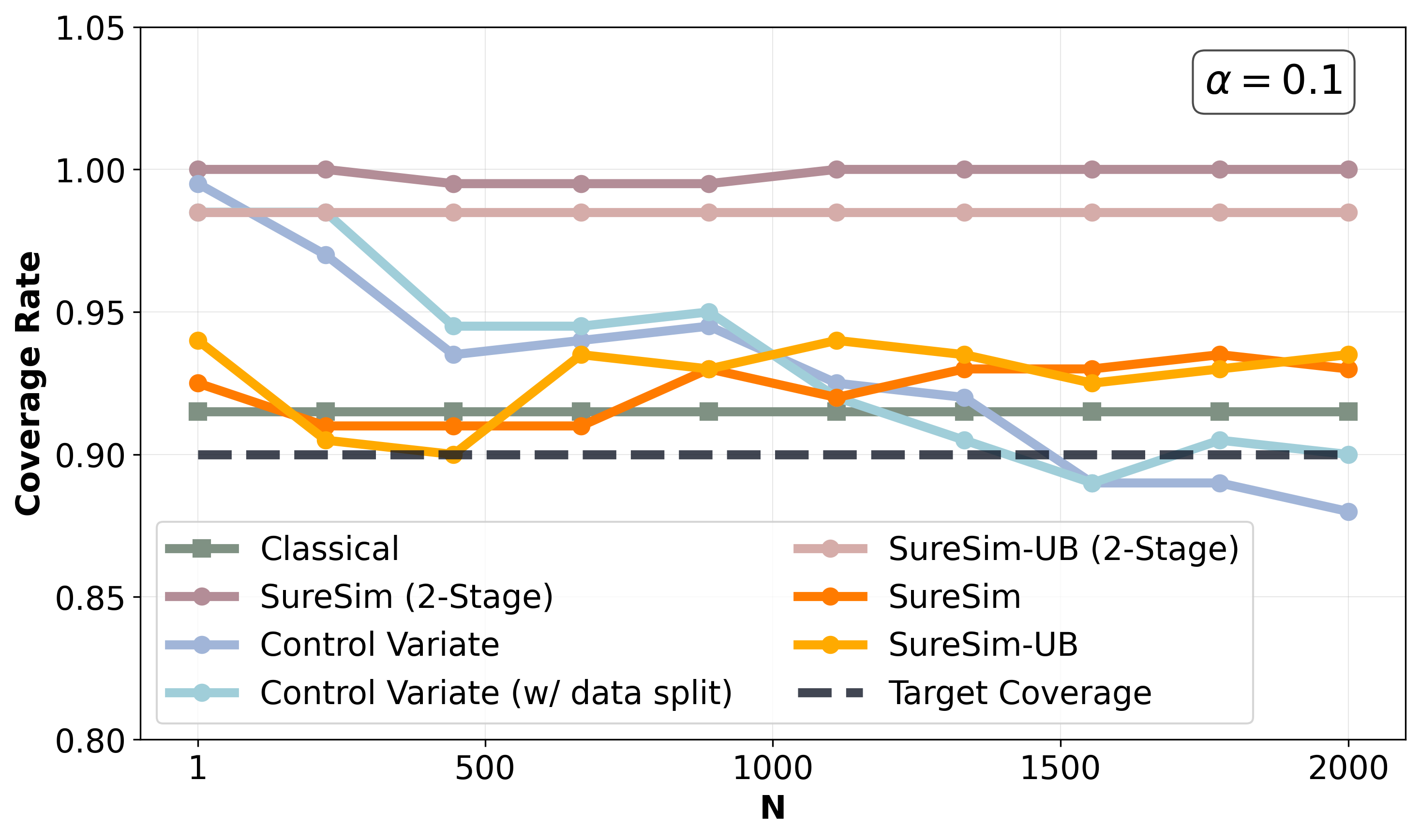}
    \caption{Validating coverage using \(400\) heldout samples.}
    \label{fig:coverage}
\end{wrapfigure}
For these experiments, we also report intervals for the \methodControlVariate method, including the standard implementation presented in~\cite{luo2025leveraging} as well as a data split version for an unbiased variance estimate. For data splitting, we use \(20\%\) of the paired data for variance estimation and the remaining for inference. As shown in~\Cref{fig:sim2sim_pi0_all}, at \(\alpha=0.1\), all methods beat \methodClassicalWSR, with the \methodPPIUnif family efficiently converging to the rectifier interval width. As the significance level decreases, \methodControlVariate no longer beats \methodClassicalWSR while \methodPPIUnif always improves.

Based on the prior discussion on the \methodControlVariate method, we do not expect it to meet the required coverage rate, and as seen in~\Cref{fig:coverage}, our experiments suggest that the empirical coverage rate can degrade with additional simulation samples. This degradation is cause for caution when interpreting the tightness of interval widths at \(\alpha=0.1\) in~\Cref{fig:sim2sim}, which highlights the importance of controlling for Type-1 error. Further rigorous validation of coverage on synthetic data is presented in~\Cref{appendix:artificial_data}\footnote{At 400 heldout samples used to characterize the ``true" mean, we can still expect some variance in the computed mean. For a very rigorous validation of coverage, we would need to use synthetic data.}.

\section{Conclusions}
\label{sec:conclusion}
We introduce \method\xspace for augmenting large-scale simulation with a relatively smaller number of real-world evaluations to provide non-asymptotically valid inferences on real-world policy performance. With a real2sim formalism, we can characterize the problem of combining real and simulation evaluations as a prediction-powered inference problem, and leverage finite-sample valid mean estimation algorithms. This pipeline allows us to evaluate the generalization capabilities of robot foundation models such as diffusion policy and \(\pi_0\). Compared to hardware-only evaluation, our method saves over \(20-25\%\) of hardware evaluations on average. 

\section{Limitations and Future Work}
\label{sec:limits}
While we introduce a non-asymptotically valid policy evaluation framework that allows us to rigorously combine real and simulation evaluations, the  simulation-real gap poses a challenge to the effectiveness of large-scale simulation. We list a few exciting directions for future work.

\textbf{Correlation between Simulation and Real.} Correcting for the bias in large-scale simulation with respect to the real-world in a statistically rigorous manner requires a paired dataset of simulation and real evaluations. The simulation–real gap also poses challenges for paired evaluations in the following ways. First, reliably predicting the real evaluation outcome for a specific initial condition is extremely challenging, as also elaborated in~\cite{pfaff2025scalable}. Secondly, real-world evaluations can be noisy --- repeated trials from the same initial condition can produce different outcomes due to inherent stochasticity and policy sensitivity to hard‑to‑control perturbations (e.g., lighting or object placement). Reducing this simulation-reality gap for evaluation along various axes of generalization (e.g., spatial, environmental factors, task) remains an important direction of future work. Additionally, data efficient methods of fine-tuning simulation evaluations to improve real-simulation correlation can be valuable.

\textbf{Action-conditioned video models.} Setting up physics-based simulations demands substantial effort and is difficult to scale or adapt to new environments and tasks. An appealing alternative is action-conditioned video world models, which can rapidly generate simulation scenes from text and image prompts~\cite{jang2025dreamgen}. As these models improve, it would be valuable to empirically assess the correlation in paired policy evaluations from world models.

\textbf{Actively sampling real environments for evaluation.} Our current framework uses random batch sampling of real environments for mean estimation. A promising extension would be to develop active sampling schemes that prioritize real evaluations in regions with larger real–simulation gaps. 

Finally, just as large-scale datasets~\cite{walke2023bridgedata,open_x_embodiment_rt_x_2023} have been pivotal for training robot foundation models, scaling evaluation will similarly require further investment in repositories of diverse manipulation tasks and environments~\cite{nasiriany2024robocasa,calli2017yale}, as well as scalable real2sim pipelines.



\section{Acknowledgments} The authors would like to thank Tijana Zrnic, Anastasios Angelopoulos, and Allen Z. Ren for insightful discussions. The authors were partially supported by the NSF CAREER Award \(\#2044149\), the Office of Naval Research (N00014-23-1-2148), and the Sloan Fellowship. A. Badithela is supported by the Presidential Postdoctoral Research Fellowship at Princeton University. 
\clearpage

\bibliographystyle{ieeetr}
\bibliography{references.bib}

\begin{thebibliography}{10}

\bibitem{kress2024robot}
H.~Kress-Gazit, K.~Hashimoto, N.~Kuppuswamy, P.~Shah, P.~Horgan, G.~Richardson, S.~Feng, and B.~Burchfiel, ``Robot learning as an empirical science: Best practices for policy evaluation,'' {\em arXiv preprint arXiv:2409.09491}, 2024.

\bibitem{gao2025taxonomy}
J.~Gao, S.~Belkhale, S.~Dasari, A.~Balakrishna, D.~Shah, and D.~Sadigh, ``A taxonomy for evaluating generalist robot policies,'' {\em arXiv preprint arXiv:2503.01238}, 2025.

\bibitem{barreiros2025careful}
J.~Barreiros, A.~Beaulieu, A.~Bhat, R.~Cory, E.~Cousineau, H.~Dai, C.-H. Fang, K.~Hashimoto, M.~Z. Irshad, M.~Itkina, {\em et~al.}, ``A careful examination of large behavior models for multitask dexterous manipulation,'' {\em arXiv preprint arXiv:2507.05331}, 2025.

\bibitem{deng2009imagenet}
J.~Deng, W.~Dong, R.~Socher, L.-J. Li, K.~Li, and L.~Fei-Fei, ``Imagenet: A large-scale hierarchical image database,'' in {\em 2009 IEEE conference on computer vision and pattern recognition}, pp.~248--255, Ieee, 2009.

\bibitem{lin2014microsoft}
T.-Y. Lin, M.~Maire, S.~Belongie, J.~Hays, P.~Perona, D.~Ramanan, P.~Doll{\'a}r, and C.~L. Zitnick, ``Microsoft coco: Common objects in context,'' in {\em European conference on computer vision}, pp.~740--755, Springer, 2014.

\bibitem{rajpurkar2016squad}
P.~Rajpurkar, J.~Zhang, K.~Lopyrev, and P.~Liang, ``Squad: 100,000+ questions for machine comprehension of text,'' {\em arXiv preprint arXiv:1606.05250}, 2016.

\bibitem{wang2018glue}
A.~Wang, A.~Singh, J.~Michael, F.~Hill, O.~Levy, and S.~R. Bowman, ``Glue: A multi-task benchmark and analysis platform for natural language understanding,'' {\em arXiv preprint arXiv:1804.07461}, 2018.

\bibitem{wang2019superglue}
A.~Wang, Y.~Pruksachatkun, N.~Nangia, A.~Singh, J.~Michael, F.~Hill, O.~Levy, and S.~Bowman, ``Superglue: A stickier benchmark for general-purpose language understanding systems,'' {\em Advances in Neural Information Processing Systems}, vol.~32, 2019.

\bibitem{snyder2025your}
D.~Snyder, A.~J. Hancock, A.~Badithela, E.~Dixon, P.~Miller, R.~A. Ambrus, A.~Majumdar, M.~Itkina, and H.~Nishimura, ``Is your imitation learning policy better than mine? policy comparison with near-optimal stopping,'' {\em arXiv preprint arXiv:2503.10966}, 2025.

\bibitem{taomaniskill3}
S.~Tao, F.~Xiang, A.~Shukla, Y.~Qin, X.~Hinrichsen, X.~Yuan, C.~Bao, X.~Lin, Y.~Liu, T.~kai Chan, Y.~Gao, X.~Li, T.~Mu, N.~Xiao, A.~Gurha, Z.~Huang, R.~Calandra, R.~Chen, S.~Luo, and H.~Su, ``Maniskill3: Gpu parallelized robotics simulation and rendering for generalizable embodied ai,'' {\em arXiv preprint arXiv:2410.00425}, 2024.

\bibitem{jang2025dreamgen}
J.~Jang, S.~Ye, Z.~Lin, J.~Xiang, J.~Bjorck, Y.~Fang, F.~Hu, S.~Huang, K.~Kundalia, Y.-C. Lin, {\em et~al.}, ``Dreamgen: Unlocking generalization in robot learning through neural trajectories,'' {\em arXiv e-prints}, pp.~arXiv--2505, 2025.

\bibitem{quevedo2025evaluating}
J.~Quevedo, P.~Liang, and S.~Yang, ``Evaluating robot policies in a world model,'' {\em arXiv preprint arXiv:2506.00613}, 2025.

\bibitem{li2024evaluating}
X.~Li, K.~Hsu, J.~Gu, K.~Pertsch, O.~Mees, H.~R. Walke, C.~Fu, I.~Lunawat, I.~Sieh, S.~Kirmani, {\em et~al.}, ``Evaluating real-world robot manipulation policies in simulation,'' {\em arXiv preprint arXiv:2405.05941}, 2024.

\bibitem{1x_world}
X.~W.~M. Team, ``1x world model: Evaluating bits, not atoms,'' tech. rep., 1X, 2025.

\bibitem{zhou2025autoeval}
Z.~Zhou, P.~Atreya, Y.~L. Tan, K.~Pertsch, and S.~Levine, ``Autoeval: Autonomous evaluation of generalist robot manipulation policies in the real world,'' {\em arXiv preprint arXiv:2503.24278}, 2025.

\bibitem{pfaff2025scalable}
N.~Pfaff, E.~Fu, J.~Binagia, P.~Isola, and R.~Tedrake, ``Scalable real2sim: Physics-aware asset generation via robotic pick-and-place setups,'' {\em arXiv preprint arXiv:2503.00370}, 2025.

\bibitem{angelopoulos2023prediction}
A.~N. Angelopoulos, S.~Bates, C.~Fannjiang, M.~I. Jordan, and T.~Zrnic, ``Prediction-powered inference,'' {\em Science}, vol.~382, no.~6671, pp.~669--674, 2023.

\bibitem{angelopoulos2023ppi++}
A.~N. Angelopoulos, J.~C. Duchi, and T.~Zrnic, ``{PPI}++: Efficient prediction-powered inference,'' {\em arXiv preprint arXiv:2311.01453}, 2023.

\bibitem{chi2023diffusion}
C.~Chi, Z.~Xu, S.~Feng, E.~Cousineau, Y.~Du, B.~Burchfiel, R.~Tedrake, and S.~Song, ``Diffusion policy: Visuomotor policy learning via action diffusion,'' {\em The International Journal of Robotics Research}, p.~02783649241273668, 2023.

\bibitem{black2024pi_0}
K.~Black, N.~Brown, D.~Driess, A.~Esmail, M.~Equi, C.~Finn, N.~Fusai, L.~Groom, K.~Hausman, B.~Ichter, {\em et~al.}, ``$\pi_0$: A vision-language-action flow model for general robot control,'' in {\em Robotics: Science and Systems}, 2025.

\bibitem{heo2023furniturebench}
M.~Heo, Y.~Lee, D.~Lee, and J.~J. Lim, ``Furniturebench: Reproducible real-world benchmark for long-horizon complex manipulation,'' in {\em Robotics: Science and Systems}, 2023.

\bibitem{luo2025fmb}
J.~Luo, C.~Xu, F.~Liu, L.~Tan, Z.~Lin, J.~Wu, P.~Abbeel, and S.~Levine, ``Fmb: a functional manipulation benchmark for generalizable robotic learning,'' {\em The International Journal of Robotics Research}, vol.~44, no.~4, pp.~592--606, 2025.

\bibitem{yang2019replab}
B.~Yang, D.~Jayaraman, J.~Zhang, and S.~Levine, ``Replab: A reproducible low-cost arm benchmark for robotic learning,'' in {\em 2019 International Conference on Robotics and Automation (ICRA)}, pp.~8691--8697, IEEE, 2019.

\bibitem{khargonkar2024scenereplica}
N.~Khargonkar, S.~H. Allu, Y.~Lu, B.~Prabhakaran, Y.~Xiang, {\em et~al.}, ``Scenereplica: Benchmarking real-world robot manipulation by creating replicable scenes,'' in {\em 2024 IEEE International Conference on Robotics and Automation (ICRA)}, pp.~8258--8264, IEEE, 2024.

\bibitem{collins2023ramp}
J.~Collins, M.~Robson, J.~Yamada, M.~Sridharan, K.~Janik, and I.~Posner, ``Ramp: A benchmark for evaluating robotic assembly manipulation and planning,'' {\em IEEE Robotics and Automation Letters}, vol.~9, no.~1, pp.~9--16, 2023.

\bibitem{pickem2017robotarium}
D.~Pickem, P.~Glotfelter, L.~Wang, M.~Mote, A.~Ames, E.~Feron, and M.~Egerstedt, ``The robotarium: A remotely accessible swarm robotics research testbed,'' in {\em 2017 IEEE International Conference on Robotics and Automation (ICRA)}, pp.~1699--1706, IEEE, 2017.

\bibitem{zhou2023train}
G.~Zhou, V.~Dean, M.~K. Srirama, A.~Rajeswaran, J.~Pari, K.~Hatch, A.~Jain, T.~Yu, P.~Abbeel, L.~Pinto, {\em et~al.}, ``Train offline, test online: A real robot learning benchmark,'' {\em arXiv preprint arXiv:2306.00942}, 2023.

\bibitem{liu2021ocrtoc}
Z.~Liu, W.~Liu, Y.~Qin, F.~Xiang, M.~Gou, S.~Xin, M.~A. Roa, B.~Calli, H.~Su, Y.~Sun, {\em et~al.}, ``Ocrtoc: A cloud-based competition and benchmark for robotic grasping and manipulation,'' {\em IEEE Robotics and Automation Letters}, vol.~7, no.~1, pp.~486--493, 2021.

\bibitem{bauer2022real}
S.~Bauer, M.~W{\"u}thrich, F.~Widmaier, A.~Buchholz, S.~Stark, A.~Goyal, T.~Steinbrenner, J.~Akpo, S.~Joshi, V.~Berenz, {\em et~al.}, ``Real robot challenge: A robotics competition in the cloud,'' in {\em NeurIPS 2021 Competitions and Demonstrations Track}, pp.~190--204, PMLR, 2022.

\bibitem{atreya2025roboarena}
P.~Atreya, K.~Pertsch, T.~Lee, M.~J. Kim, A.~Jain, A.~Kuramshin, C.~Eppner, C.~Neary, E.~Hu, F.~Ramos, {\em et~al.}, ``Roboarena: Distributed real-world evaluation of generalist robot policies,'' {\em arXiv preprint arXiv:2506.18123}, 2025.

\bibitem{khazatsky2024droid}
A.~Khazatsky, K.~Pertsch, S.~Nair, A.~Balakrishna, S.~Dasari, S.~Karamcheti, S.~Nasiriany, M.~K. Srirama, L.~Y. Chen, K.~Ellis, {\em et~al.}, ``Droid: A large-scale in-the-wild robot manipulation dataset,'' {\em arXiv preprint arXiv:2403.12945}, 2024.

\bibitem{todorov2012mujoco}
E.~Todorov, T.~Erez, and Y.~Tassa, ``Mujoco: A physics engine for model-based control,'' in {\em 2012 IEEE/RSJ international conference on intelligent robots and systems}, pp.~5026--5033, IEEE, 2012.

\bibitem{tassa2018deepmind}
Y.~Tassa, Y.~Doron, A.~Muldal, T.~Erez, Y.~Li, D.~d.~L. Casas, D.~Budden, A.~Abdolmaleki, J.~Merel, A.~Lefrancq, {\em et~al.}, ``Deepmind control suite,'' {\em arXiv preprint arXiv:1801.00690}, 2018.

\bibitem{makoviychuk2021isaac}
V.~Makoviychuk, L.~Wawrzyniak, Y.~Guo, M.~Lu, K.~Storey, M.~Macklin, D.~Hoeller, N.~Rudin, A.~Allshire, A.~Handa, {\em et~al.}, ``Isaac gym: High performance gpu-based physics simulation for robot learning,'' {\em arXiv preprint arXiv:2108.10470}, 2021.

\bibitem{zhu2020robosuite}
Y.~Zhu, J.~Wong, A.~Mandlekar, R.~Mart{\'\i}n-Mart{\'\i}n, A.~Joshi, S.~Nasiriany, and Y.~Zhu, ``robosuite: A modular simulation framework and benchmark for robot learning,'' {\em arXiv preprint arXiv:2009.12293}, 2020.

\bibitem{nasiriany2024robocasa}
S.~Nasiriany, A.~Maddukuri, L.~Zhang, A.~Parikh, A.~Lo, A.~Joshi, A.~Mandlekar, and Y.~Zhu, ``Robocasa: Large-scale simulation of everyday tasks for generalist robots,'' {\em arXiv preprint arXiv:2406.02523}, 2024.

\bibitem{rlbench}
S.~James, Z.~Ma, D.~R. Arrojo, and A.~J. Davison, ``Rlbench: The robot learning benchmark \& learning environment,'' {\em IEEE Robotics and Automation Letters}, vol.~5, no.~2, pp.~3019--3026, 2020.

\bibitem{pumacay2024colosseum}
W.~Pumacay, I.~Singh, J.~Duan, R.~Krishna, J.~Thomason, and D.~Fox, ``The colosseum: A benchmark for evaluating generalization for robotic manipulation,'' {\em arXiv preprint arXiv:2402.08191}, 2024.

\bibitem{zheng2022vlmbench}
K.~Zheng, X.~Chen, O.~C. Jenkins, and X.~Wang, ``Vlmbench: A compositional benchmark for vision-and-language manipulation,'' {\em Advances in Neural Information Processing Systems}, vol.~35, pp.~665--678, 2022.

\bibitem{mees2022calvin}
O.~Mees, L.~Hermann, E.~Rosete-Beas, and W.~Burgard, ``Calvin: A benchmark for language-conditioned policy learning for long-horizon robot manipulation tasks,'' {\em IEEE Robotics and Automation Letters}, vol.~7, no.~3, pp.~7327--7334, 2022.

\bibitem{yang2023learning}
M.~Yang, Y.~Du, K.~Ghasemipour, J.~Tompson, D.~Schuurmans, and P.~Abbeel, ``Learning interactive real-world simulators,'' {\em arXiv preprint arXiv:2310.06114}, vol.~1, no.~2, p.~6, 2023.

\bibitem{brooks2024video}
T.~Brooks, B.~Peebles, C.~Holmes, W.~DePue, Y.~Guo, L.~Jing, D.~Schnurr, J.~Taylor, T.~Luhman, E.~Luhman, {\em et~al.}, ``Video generation models as world simulators,'' {\em OpenAI Blog}, vol.~1, no.~8, p.~1, 2024.

\bibitem{agarwal2025cosmos}
N.~Agarwal, A.~Ali, M.~Bala, Y.~Balaji, E.~Barker, T.~Cai, P.~Chattopadhyay, Y.~Chen, Y.~Cui, Y.~Ding, {\em et~al.}, ``Cosmos world foundation model platform for physical ai,'' {\em arXiv preprint arXiv:2501.03575}, 2025.

\bibitem{kadian2020sim2real}
A.~Kadian, J.~Truong, A.~Gokaslan, A.~Clegg, E.~Wijmans, S.~Lee, M.~Savva, S.~Chernova, and D.~Batra, ``Sim2real predictivity: Does evaluation in simulation predict real-world performance?,'' {\em IEEE Robotics and Automation Letters}, vol.~5, no.~4, pp.~6670--6677, 2020.

\bibitem{majumdar2025predictive}
A.~Majumdar, M.~Sharma, D.~Kalashnikov, S.~Singh, P.~Sermanet, and V.~Sindhwani, ``Predictive red teaming: Breaking policies without breaking robots,'' {\em arXiv preprint arXiv:2502.06575}, 2025.

\bibitem{o2018scalable}
M.~O'Kelly, A.~Sinha, H.~Namkoong, R.~Tedrake, and J.~Duchi, ``Scalable end-to-end autonomous vehicle testing via rare-event simulation,'' {\em Advances in Neural Information Processing Systems}, vol.~31, 2018.

\bibitem{vincent2024generalizable}
J.~A. Vincent, H.~Nishimura, M.~Itkina, P.~Shah, M.~Schwager, and T.~Kollar, ``How generalizable is my behavior cloning policy? a statistical approach to trustworthy performance evaluation,'' {\em IEEE Robotics and Automation Letters}, 2024.

\bibitem{hoeffding1963probability}
W.~Hoeffding, ``Probability inequalities for sums of bounded random variables,'' {\em Journal of the American statistical association}, vol.~58, no.~301, pp.~13--30, 1963.

\bibitem{mandyam2025perry}
A.~Mandyam, J.~Meng, G.~Gao, J.~Sun, M.~Schwager, B.~E. Engelhardt, and E.~Brunskill, ``Perry: Policy evaluation with confidence intervals using auxiliary data,'' {\em arXiv preprint arXiv:2507.20068}, 2025.

\bibitem{luo2025leveraging}
R.~Luo, H.~Yang, M.~Watson, A.~Sharma, S.~Veer, E.~Schmerling, and M.~Pavone, ``Leveraging correlation across test platforms for variance-reduced metric estimation,'' {\em arXiv preprint arXiv:2506.20553}, 2025.

\bibitem{waudby2024estimating}
I.~Waudby-Smith and A.~Ramdas, ``Estimating means of bounded random variables by betting,'' {\em Journal of the Royal Statistical Society Series B: Statistical Methodology}, vol.~86, no.~1, pp.~1--27, 2024.

\bibitem{zrnic2024active}
T.~Zrnic and E.~Candes, ``Active statistical inference,'' in {\em International Conference on Machine Learning}, pp.~62993--63010, PMLR, 2024.

\bibitem{deitke2023objaverse}
M.~Deitke, D.~Schwenk, J.~Salvador, L.~Weihs, O.~Michel, E.~VanderBilt, L.~Schmidt, K.~Ehsani, A.~Kembhavi, and A.~Farhadi, ``Objaverse: A universe of annotated 3d objects,'' in {\em Proceedings of the IEEE/CVF conference on computer vision and pattern recognition}, pp.~13142--13153, 2023.

\bibitem{calli2017yale}
B.~Calli, A.~Singh, J.~Bruce, A.~Walsman, K.~Konolige, S.~Srinivasa, P.~Abbeel, and A.~M. Dollar, ``Yale-cmu-berkeley dataset for robotic manipulation research,'' {\em The International Journal of Robotics Research}, vol.~36, no.~3, pp.~261--268, 2017.

\bibitem{walke2023bridgedata}
H.~R. Walke, K.~Black, T.~Z. Zhao, Q.~Vuong, C.~Zheng, P.~Hansen-Estruch, A.~W. He, V.~Myers, M.~J. Kim, M.~Du, {\em et~al.}, ``Bridgedata v2: A dataset for robot learning at scale,'' in {\em Conference on Robot Learning}, pp.~1723--1736, PMLR, 2023.

\bibitem{open_x_embodiment_rt_x_2023}
``Open {X-E}mbodiment: Robotic learning datasets and {RT-X} models.'' \url{https://arxiv.org/abs/2310.08864}, 2023.

\bibitem{he2015deepresiduallearningimage}
K.~He, X.~Zhang, S.~Ren, and J.~Sun, ``Deep residual learning for image recognition,'' 2015.

\end{thebibliography}

\clearpage
\beginappendix{


\section{Summary Statistics from Experiments}
\label{sec:appendix_tab}
\begin{table}[h]
    \centering
    \small   
    \begin{tabular}{l|cccccccc}
        \toprule
        {Experiment}  & $n$ & $N$ &  $\rho$ & \(\frac{1}{n}\sum_{i=1}^n Y_i\) & \(\frac{1}{n}\sum_{i=1}^n f(\tilde{X}_i)\) & \(\frac{1}{N}\sum_{i=1}^N f(\tilde{X}_i)\) & \(\hat{\sigma}^2_{Y}\) & \(\hat{\sigma}^2_{Y - f}\) \\ 
        \midrule
        Diffusion Policy (Real2Sim) & 60 & 700 & 0.702 & 0.246 & 0.188 & 0.174 & 0.104 & 0.054 \\
        \(\pi_0\) (Real2Sim, moderate \(\rho\)) & 60 & 2100 & 0.588 & 0.825 & 0.820 & 0.772 & 0.138 & 0.090 \\
        \(\pi_0\) (Real2Sim, low \(\rho\)) & 60 & 2100 & -0.051 & 0.983 & 0.932 & 0.928 & 0.014 & 0.029\\
        \(\pi_0\) (Sim2Sim, high \(\rho\)) & 100 & 2000 & 0.974 & 0.751 & 0.731 & 0.732 & 0.116 & 0.006 \\
        \bottomrule
    \end{tabular}
    \caption{Summary statistics indicating the number of paired trials \(n\), additional simulation evaluations \(N\), Pearson correlation coefficient \(\rho\), sample real mean \(\frac{1}{n}\sum_{i=1}^n Y_i\), sample paired simulation mean \(\frac{1}{n}\sum_{i=1}^n f(\tilde{X}_i)\), sample additional simulation mean \(\frac{1}{N}\sum_{i=1}^N f(\tilde{X}_i)\), the sample real variance \(\hat{\sigma}^2_{Y}\), and the sample rectifier variance \(\hat{\sigma}^2_{Y - f}\). The reported statistics are averaged over \(100\) draws of data.}
    \label{tab:exp}
\end{table}
\section{Algorithms}

\begin{algorithm2e}[H]
\caption{Uniform Prediction Powered Inference (\textsc{UniformPPI})}
\label{alg:uniform_ppi}
\KwIn{Paired dataset $D_{\text{paired}}$, Simulation dataset $D_{\text{sim}}$, Sim outcomes $f$, counts $n, N$, significance level $\alpha$}
\KwOut{Confidence interval $CI$}
    \For{$i \gets 1$ \KwTo $n+N$}{
        \(\xi_i = 1\) if $X_i$ has a real evaluation, else $\xi_i=0$ \\
        \(\Delta_i = f(\tilde{X}_i) + \frac{n+N}{n}(Y_i - f(\tilde{X}_i))\cdot \xi_i\)
    }
    \(D_{\text{unif}} = \{\Delta_i\}_{i=1}^{n+N}\) \tcp*[f]{Problem Data}\\
    \(CI \gets \textsc{WSR}(D_{\text{unif}},\, \alpha = \alpha, L=-\tfrac{n+N}{n}, U=1 + \tfrac{n+N}{n})\) \tcp*[f]{Single call to WSR}\\
\Return{$CI$}
\end{algorithm2e}

\begin{algorithm2e}[H]
\caption{Two-Stage Prediction Powered Inference (\textsc{2-StagePPI})}
\label{alg:standard_ppi}
\KwIn{Paired dataset $D_{\text{paired}}$, Simulation dataset $D_{\text{sim}}$, levels $\alpha, \delta$}
\KwOut{Confidence interval $CI$}
    \((f_l, f_u) \gets \textsc{WSR}(D_{\text{sim}},\, \alpha = \delta, L=0, U=1)\) \tcp*[f]{Additional Simulation Confidence Interval} \\
    \For{$i \gets 1$ \KwTo $n$}{
        \(\Delta_i = Y_i - f(\tilde{X}_i)\)
    }
    \((R_l, R_u) \gets \textsc{WSR}(\{\Delta_i\}_{i=1}^n,\, \alpha = \alpha - \delta, L=-1, U=1)\) \tcp*[f]{Rectifier Confidence Interval} \\
    \(CI \gets (f_l -R_u, f_u -R_l)\) \tcp*[f]{Union bound} \\
\Return{$CI$}
\end{algorithm2e}

\begin{algorithm2e}[H]
\caption{Non-asymptotic mean estimation via Waudby-Smith Ramdas (WSR) Procedure~\cite{angelopoulos2023prediction,waudby2024estimating}} 
\label{alg:meanCI}
\KwData{Data points $\{Z_1,\dots,Z_n\}$, error level $\alpha \in (0,1)$, range $[L,U]$ such that $Z_i \in [L,U]$.}
\KwResult{Confidence interval $CI$ for the mean}
    \For{$i \gets 1$ \KwTo $n$}{
        $Z_i \gets (Z_i - L)/(U-L)$ \tcp*[f]{Normalize to $[0,1]$}
    }
    Construct fine grid $M_{\rm grid}$ over $[0,1]$ \\
    Initialize set of candidate means $\mathcal A \gets M_{\rm grid}$ \\
    
    \For{$t \gets 1$ \KwTo $n$}{
        $\hat{\mu}_t \gets \dfrac{0.5 + \sum_{j=1}^t Z_j}{t+1}$ \\
        $\hat{\sigma}_t^2 \gets \dfrac{0.25 + \sum_{j=1}^t (Z_j - \hat{\mu}_t)^2}{t+1}$ \\
        $\lambda_t \gets \sqrt{\dfrac{2\log(2/\alpha)}{n\hat{\sigma}_{t-1}^2}}$ \\
        
        \For{$m \in  M_{\rm grid}$}{
            \Comment{In computing the martingales, we choose the hyperparameter \(c = 0.99\) due to its empirical performance} \\ 
            $M_t^+(m) \gets \Bigl(1+\min\bigl(\lambda_t, \tfrac{c}{m}\bigr)(Z_t - m)\Bigr) M^+_{t-1}(m)$  \\
            $M_t^-(m) \gets \Bigl(1-\min\bigl(\lambda_t, \tfrac{c}{1-m}\bigr)(Z_t - m)\Bigr) M^-_{t-1}(m)$ \\
            $M_t(m) \gets \tfrac{1}{2} \max\{M_t^+(m), M_t^-(m)\}$ \tcp*[f]{Martingale} \\
            
            \If{$M_t(m) \geq 1/\alpha$}{
                $\mathcal A \gets \mathcal A \setminus \{m\}$ \tcp*[f]{Remove $m$ from set of candidate means}
            }
        }
    }
$C_\alpha = \{m(U-L)+L : m \in \mathcal A\}$ \tcp*[f]{True mean lies in this set with high probability}\\
$CI = [\max\{0, \min{C_{\alpha}}\}, \min\{1, \max{C_{\alpha}}\}]$ \\
\Return{CI}
\end{algorithm2e}

\section{Choosing the Confidence For Each Interval}
 The allocation of risk between the rectifier and simulator intervals can be approximately optimized in an efficient manner. The sub-Gaussian nature of the respective means ensures that the interval growth rate with respect to increasing confidence is monotonic and increasing. Thus, there is an approximate equilibrium that can be found via binary search, in which the rate of width increase in the rectifier (resp., simulator) is precisely offset by the rate of interval shrinkage in the simulator (resp., rectifier). On either ``side'' of this equilibrium, the rate of growth of one of the intervals outpaces the rate of shrinkage of the other, making the landscape approximately convex. Generally speaking, the optimal allocation in practical situations has $\delta \approx 0.9\alpha$.

\section{Evaluation on Artificial Data}
\label{appendix:artificial_data}
In order to investigate counterfactual properties of the evaluation methods, we test all method using artificial (simulated) data with known statistical properties. This is \textbf{not} data generated by a physics-based \emph{robot simulator}, but is rather simulated i.i.d. draws of \emph{scalar random variables} with known statistical properties. We term this data ``artificial'' in order to avoid any confusion with the the simulator predictions in \Cref{sec:experiments}. 

\subsection{Value of Artificial Data and Research Questions}
\label{artificial_data:justification}
Practical estimation problems arise precisely because the investigator does not have access to the true statistical parameter in question (in this case, the mean performance). Thus, when evaluating on real data as in \Cref{sec:experiments}, we cannot verify whether the confidence intervals we generate -- or those generated by any baseline procedure -- actually contain the true mean. As such, using artificial data allows for the important step of verifying the theoretical claims of each method in practice, so that they may profitably be used on such problems as may be encountered in, for example, the sciences and engineering. Furthermore, access to the ``true parameter labels'' for artificial data allow us to efficiently pose hundreds or even thousands of estimation problems reflective of varying contexts, which inform the reader as to the best method for their particular application. 

The core additional technical objection this must raise is the degree to which the simulated data fails to represent some data that may be observed by the practitioner; necessarily, it is impossible to sample from, and validate against, \emph{all distributions} -- certainly in finite time. Addressing this problem will be crucial to effective characterization and evaluation. 

To the aforementioned ends, we provide the following core analyses via the evaluations on artificial data:
\begin{itemize}
    \item A justification of the generality of our data generation process with respect to key parameters;
    \item A discussion of the most informative metrics in evaluating estimation procedures;
    \item An investigation of the effectiveness of all methods subject to variation in the key parameters;
    \item A brief discussion and usage guide for the strengths of each method, and interpretable scenarios in which one should likely be preferred to the others. 
\end{itemize}

\subsection{The Key Parameters and Data Generation}
\label{artificial_data:generation}
As introduced in \Cref{sec:problem}, the robot evaluation problem tackled here is a special case of a more general mean estimation problem in statistics. The canonical Neyman-Pearson framework for understanding these estimation problems relies on several key parameters: the batch size ($n$) and the significance level ($\alpha$). As introduced in \Cref{sec:ppi}, we are interested in using the information contained in proxy signals (e.g., simulators) to effectively increase the sample size of real evaluations. Thus, this procedure \emph{also depends} on the amount of proxy data ($N$) and the degree to which the simulator is ``useful'' -- informally, the amount of additional information contained in the proxy variables. 

This last piece of information is of course key to the investigation. We argue that, consistent with the analysis of control variates methods, the critical measure by which the proxy variable improves nonasymptotic (finite-sample) confidence interval generation is in the variance reduction of the (unbiased) mean estimator. Intuitively, such a reduction tightens the confidence intervals while ensuring Type-1 error control at all data scales. With this in mind, we use as the core ``effectiveness measure'' the Pearson correlation coefficient, $\rho$, which is a direct ratio of the real-to-proxy covariance to their geometric mean variance. This intuition is reflected directly in the control variates analysis of \cite{luo2025leveraging}, particularly with respect to their Theorem 1. 

Given the preceding discussion, we intend to investigate the relative advantages of each estimation procedure as a function of the four stated parameters: $\alpha$, $n$, $N$, and $\rho$. To generate artificial data, we construct artificial real data of size $k(n+N)$ drawn uniformly in the interval $[\max\{0, 2\mu-1\}, \min\{ 2\mu, 1\}]$. This enforces a tunable true population mean $\mu$ while ensuring that the data is always bounded in $[0, 1]$.\footnote{To avoid unnecessary subtleties around the effects of interval truncation at $0$ and $1$ on the aggregate interval width metrics, we will in general set the means to be equal to $0.5$.} The proxy data requires a desired value $\rho^*$. To generate the artificial proxy data, the real data is copied, shifted to mean $\mu_{sim}$, and then iteratively perturbed by small amounts of random noise or small perfect-signal gradients in order to push the empirical correlation to $\rho$. Matched and unmatched datasets of respective size $n$ and $N$ are drawn from partitions of the large dataset; for the unmatched data, the real labels are discarded for the purposes of running the algorithms. To save time, this single large dataset can be sampled from repeatedly (i.e., bootstrapped), or new datasets can be generated for each experiment. We opt for the latter, though it is more time-consuming in practice for empirically negligible effects. 

\textbf{A Brief Discussion of Estimation Metrics}
\label{artificial_data:what_metric}
The ``proper'' metrics to report for the problem of mean estimation admits a wide array of context-relevant options. We argue for the metrics herein, and attempt to briefly justify the preference. 

\textbf{The Natural Option: Interval Widths}
The ultimate purpose of estimation in our context is to minimize the region of uncertainty in which the true parameter lies (subject to a tunable risk of error); this is a dual result to many ``operationalizable" uses, including tests of maximal efficiency and power, certification in the least number of trials, etc. As such, it is unsurprising that reporting interval width directly is the most natural metric, and is our primary metric of choice in this work. 

\textbf{A Caution About Variance Minimization}
Another natural metric, albeit one slightly upstream of the intended methodological usage, is estimator variance. This analysis is interchangeable with the interval width (i.e., equivalent under monotonic transformations), but \emph{only when the space of estimators is constrained to be unbiased}. Unbiased estimators overwhelmingly dominate among methods used in practice, but analysis can be misleading when the constraint is not satisfied. A minimum-variance estimator is \emph{essentially meaningless} (for example, the estimator `5' has no variance over the draw of the data); a minimum-variance \emph{unbiased} estimator, on the other hand, can be exceedingly novel and useful. 

\textbf{Dependence on Significance Level}
Bounded random variables are a special case of random variables with bounded moments. These random variables are sub-Gaussian, and therefore any optimal interval widths (across data scales) should be able to attain poly-logarithmic dependence on the significance level.

\subsection{Results on Artificial Data}
\label{app:artificial_data:results}
We illustrate the three aforementioned themes in evaluation over artificial data. All results will report interval widths as the primary metric (Theme 1), and will discuss the downside of additional metrics via the example cases. Second, the limitations of variance minimization will be illustrated (Theme 2), which will also inform our investigation of each method's coverage (i.e., enforcing the Type-I error control constraint in \Cref{eq:guarantee}). Finally, we will sketch the gap in efficiency with respect to confidence or significance level, which has implications for different types of validation settings encountered in practice. To be specific, we will discuss in particular the implications of statistical guarantees in safety-critical certification and evaluation paradigms. 

In order to avoid certain distracting or confounding phenomena, we present intervals for data with characteristics designed to highlight the fundamental behavior of the algorithms. Specifically: the real data and simulator data means are set arbitrarily to 0.5 each, in order to minimize instances of truncation of the intervals at 0 or 1. Truncation does not in general benefit any particular method, but it does increase variation in interval widths that can make the results noisier. As these results are purely designed to validate existential and not universal quantifications of algorithm behavior, this choice does not bias the result and discussion. 

\textbf{Interval Width as Simulator Data or Correlation Grows}
\label{app:artificial_data:width_vs_nsim}
We begin by validating the behavior of each algorithm seen in \Cref{sec:experiments}. The artificial data is iteratively redrawn for $n=100$ and varying levels of $N$ up to $10$k additional simulation runs. We first consider a case of relative strength for \methodControlVariate, taking a large correlation $\rho = 0.97$ and varying $\alpha$ across approximately two orders of magnitude. As shown in \Cref{fig:width_vs_nsim}, every method has near-monotonic improvement (in expectation) as the amount of additional sim data grows, reflecting the intuition that there must be more `information' being given to the evaluator. However, the intervals do not asymptote to zero, indicating that, even so, there remains fundamental uncertainty in linking the sim data to the real data (the rectifier uncertainty) that is \emph{irreduceable} given a fixed amount of real data. In other words, we cannot trust the sim data to an arbitrary degree, but can still use the data productively to tighten the intervals. As shown in~\Cref{fig:width_vs_nsim,fig:width_vs_rho}, the \methodPPIUnif family of methods recover the real-simulation gap determined by the number of paired trials \(n\). In contrast, the two-stage methods, \methodPPINonAsym and \methodUB, remain inefficient even at \(N=10000\), requiring an even higher number of simulation evaluations to converge.

\begin{figure}[htbp]
    \centering
    \begin{subfigure}[b]{0.48\textwidth}
        \centering
        \includegraphics[width=\textwidth]{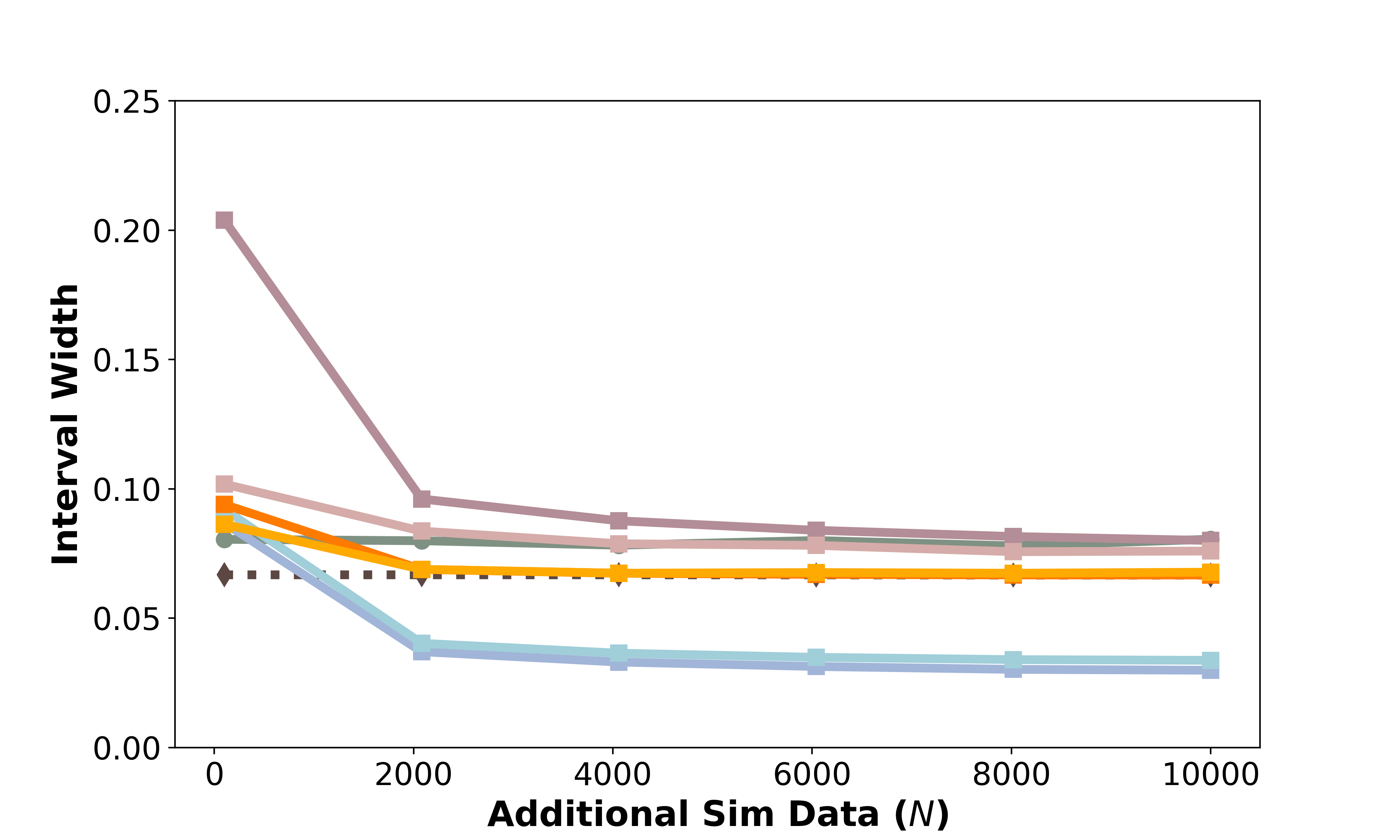}
        \caption{Interval width vs ($N$): $\alpha = 0.1$, $\rho=0.97$}
        \label{fig:width_vs_nsim_alpha_0.1}
    \end{subfigure}
    \hfill
    \begin{subfigure}[b]{0.48\textwidth}
        \centering
        \includegraphics[width=\textwidth]{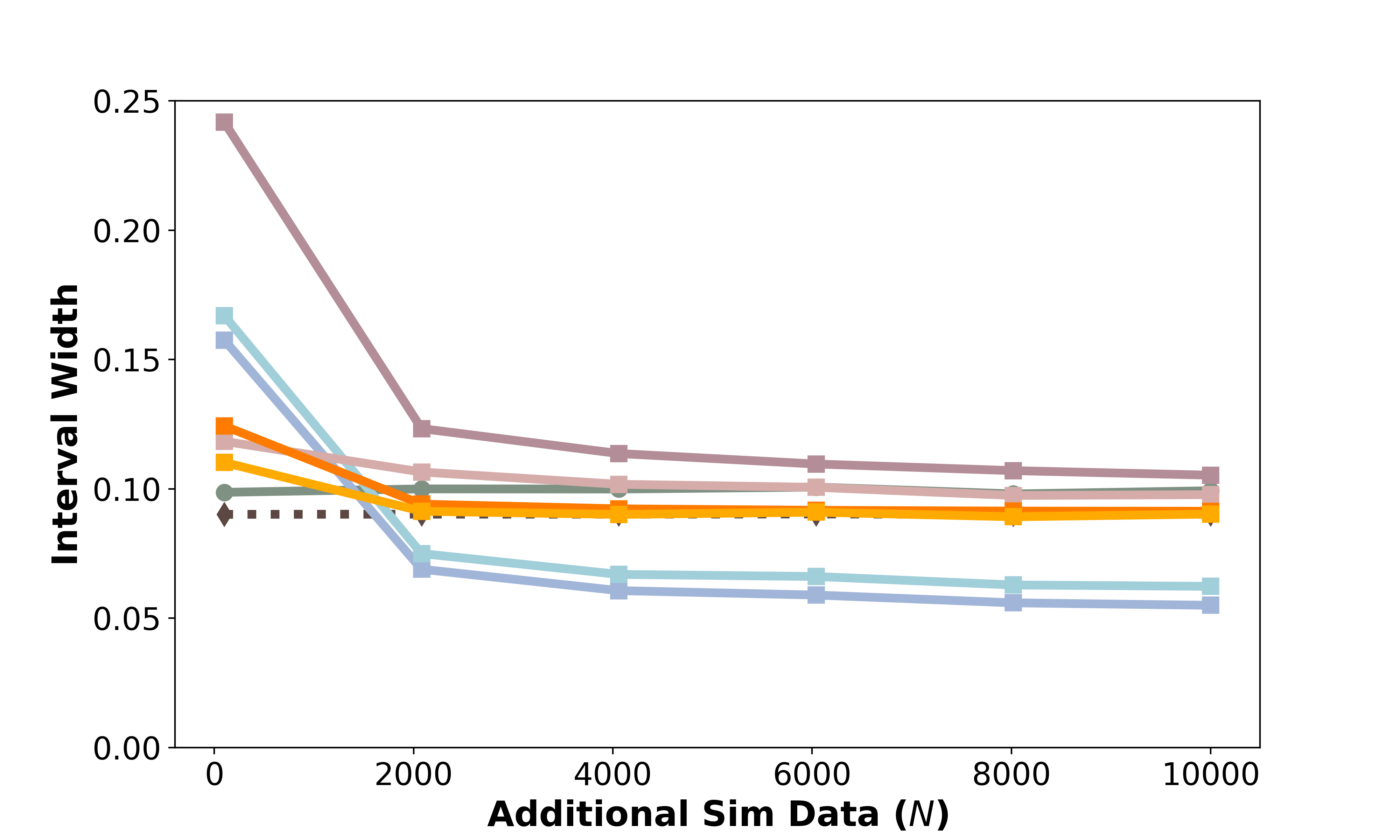}
        \caption{Interval width vs ($N$): $\alpha = 0.03$, $\rho=0.97$}
        \label{fig:width_vs_nsim_alpha_0.03}
    \end{subfigure} \\
    \begin{subfigure}[b]{0.48\textwidth}
        \centering
        \includegraphics[width=\textwidth]{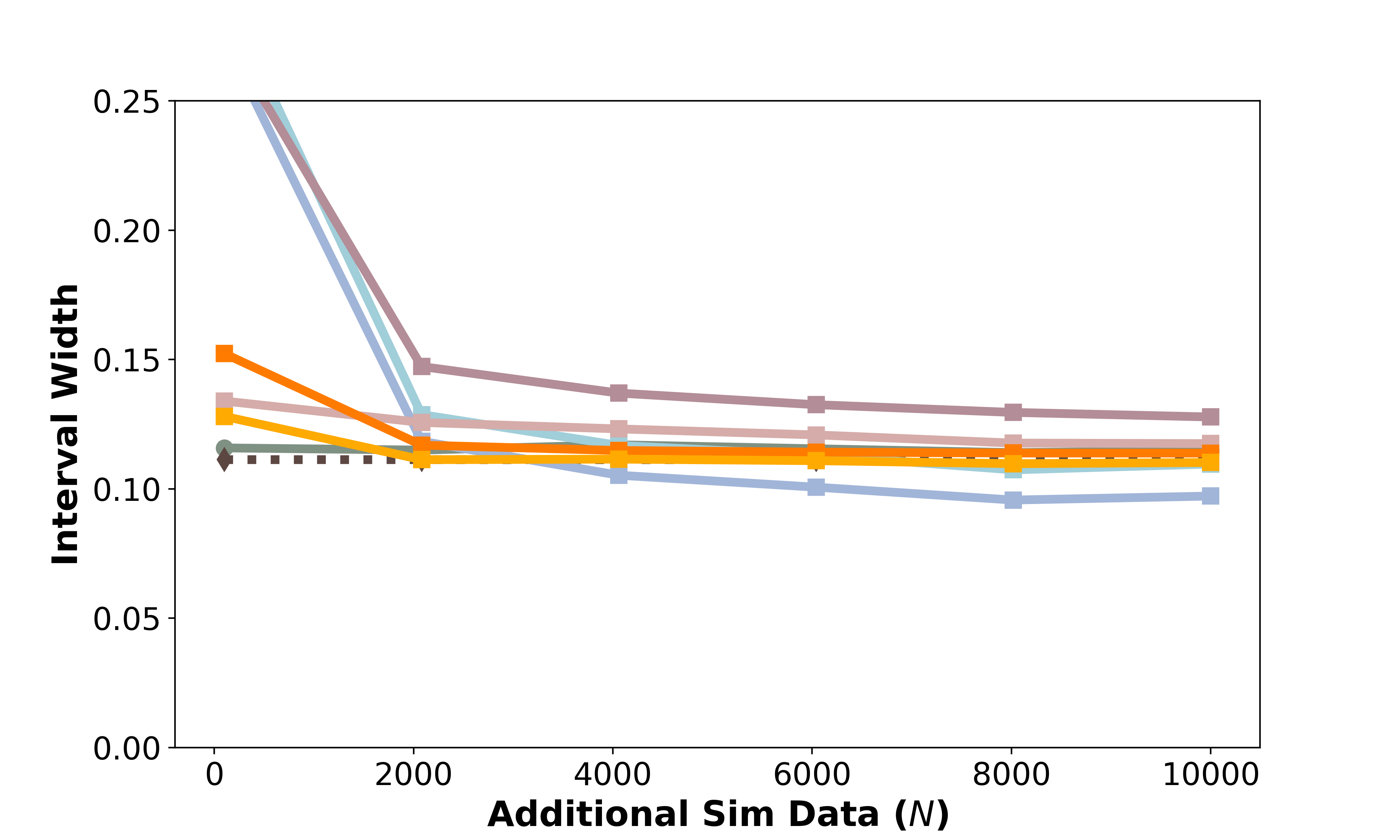}
        \caption{Interval width vs ($N$): $\alpha = 0.01$, $\rho=0.97$}
        \label{fig:width_vs_nsim_alpha_0.01}
    \end{subfigure}
    \hfill
    \begin{subfigure}[b]{0.48\textwidth}
        \centering
        \includegraphics[width=\textwidth]{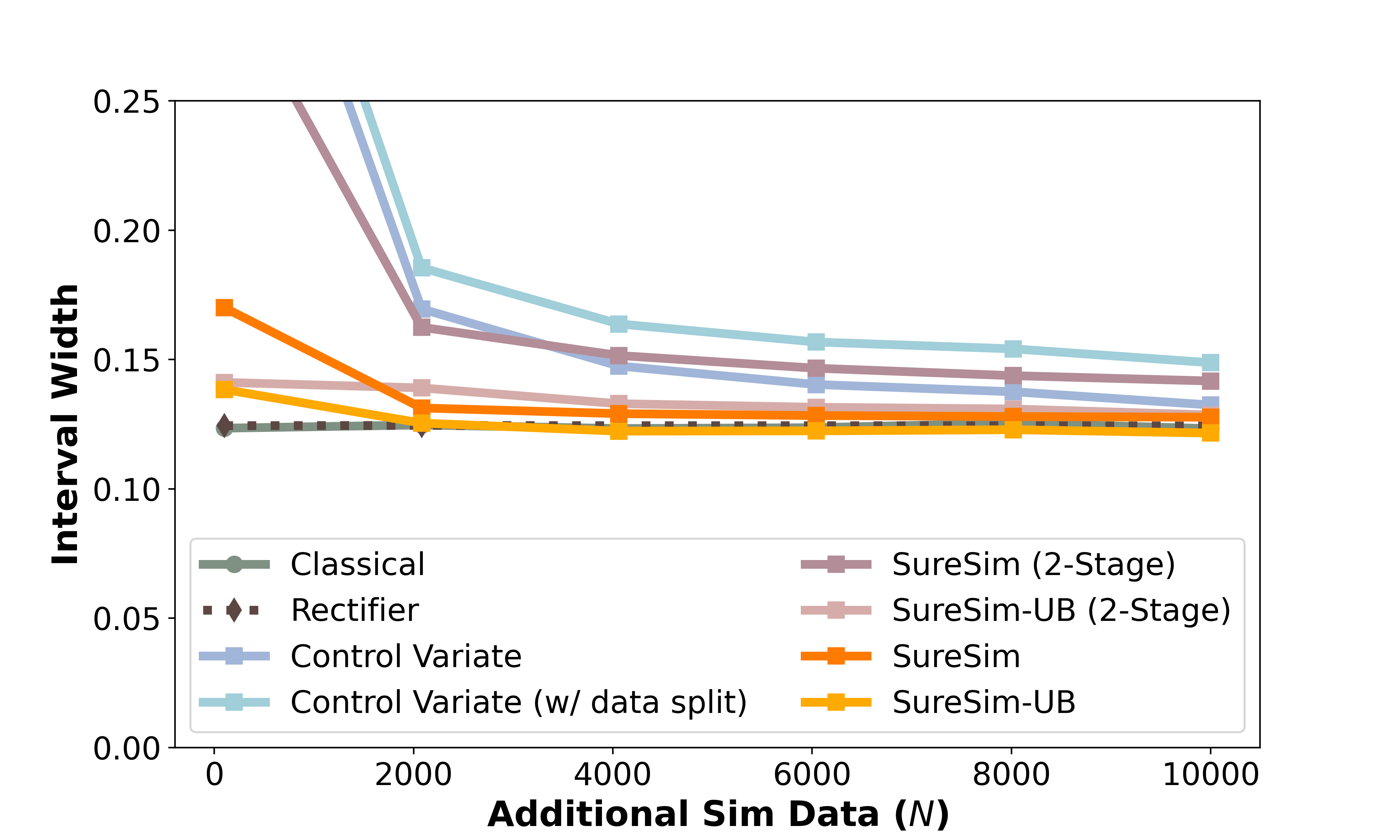}
        \caption{Interval width vs ($N$): $\alpha = 0.005$, $\rho=0.97$}
        \label{fig:width_vs_nsim_alpha_0.004}
    \end{subfigure}
    \caption{Interval widths on all methods for artificial data. Each plot shows the width against varying $N$ ($n = 100$). Results averaged over 100 independent redraws of data. The desired confidence level increases ($\alpha$ decreases) left-to-right, top-to-bottom. Note that \methodControlVariate constructs tighter intervals at large $\alpha$, but that the interval widths are much more sensitive as $\alpha$ changes. As will be shown in \Cref{fig:coverage_vs_nsim}, the biased nature of the CV estimator leads to miscoverage in regimes for which its intervals appear to be narrower. }
    \label{fig:width_vs_nsim}
\end{figure}
In \Cref{fig:width_vs_rho}, we generalize these results to variations across the true correlation between real data and simulation. Naturally, higher correlation implies greater signal in the proxy (simulation) data, and therefore more achievable tightening. This also validates the analysis of monotonic and quadratic \methodControlVariate interval width scaling given in \cite{luo2025leveraging}. Note that the numbers in \Cref{fig:width_vs_nsim} correspond to nearly the right-most points of the curves in \Cref{fig:width_vs_rho}.\footnote{This statement is modulo the small differences in $\alpha$ for two of the plots, which differed in order to allow us to illustrate qualitatively different coverage behavior for \methodControlVariate in \Cref{fig:coverage_vs_nsim}.} Additionally, even if the 

\begin{figure}[h]
    \centering
    \begin{subfigure}[b]{0.48\textwidth}
        \centering
        \includegraphics[width=\textwidth]{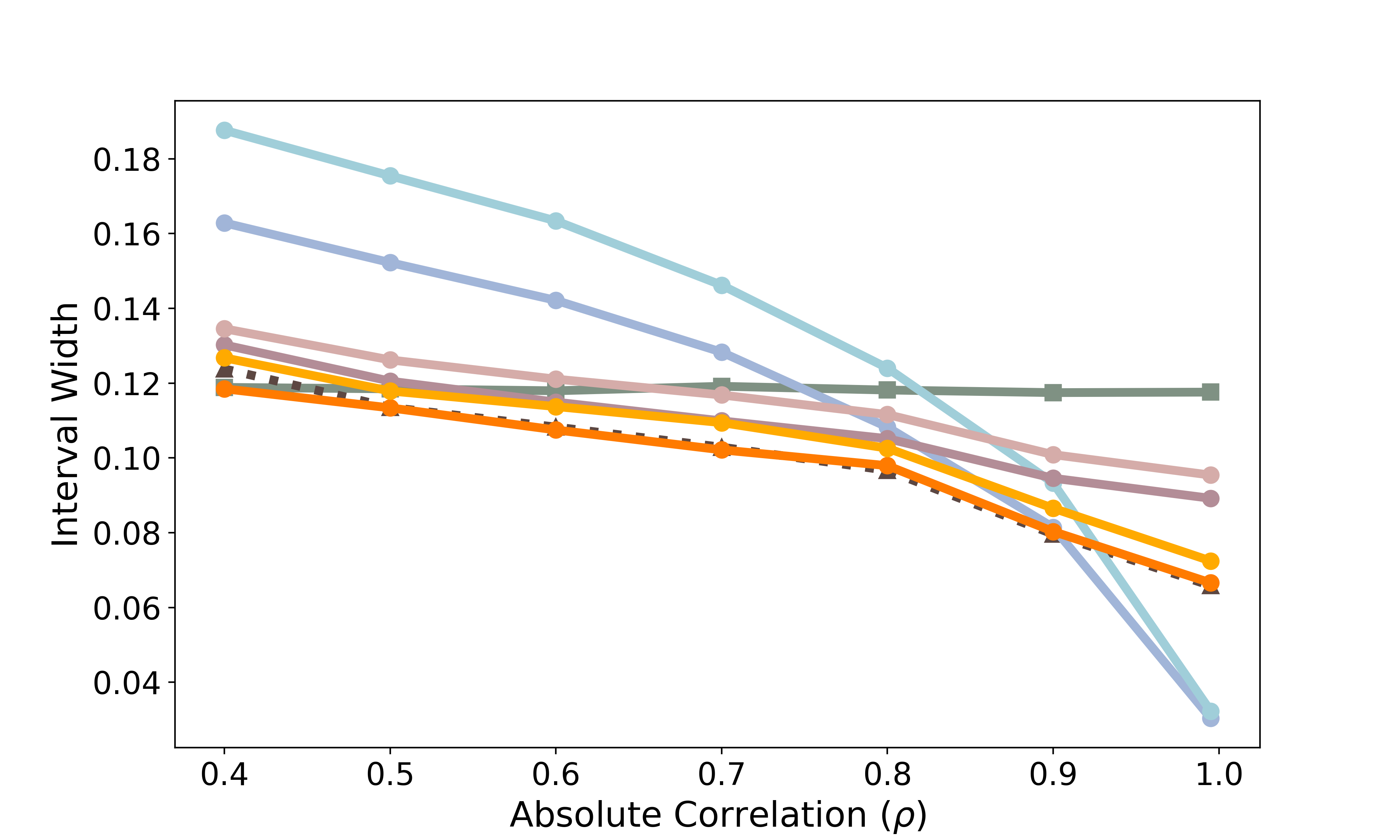}
        \caption{Interval width vs ($\rho$); $n=100$, $N=5000$, $\alpha = 0.1$}
        \label{fig:width_vs_rho_alpha_0.1}
    \end{subfigure}
    \hfill
    \begin{subfigure}[b]{0.48\textwidth}
        \centering
        \includegraphics[width=\textwidth]{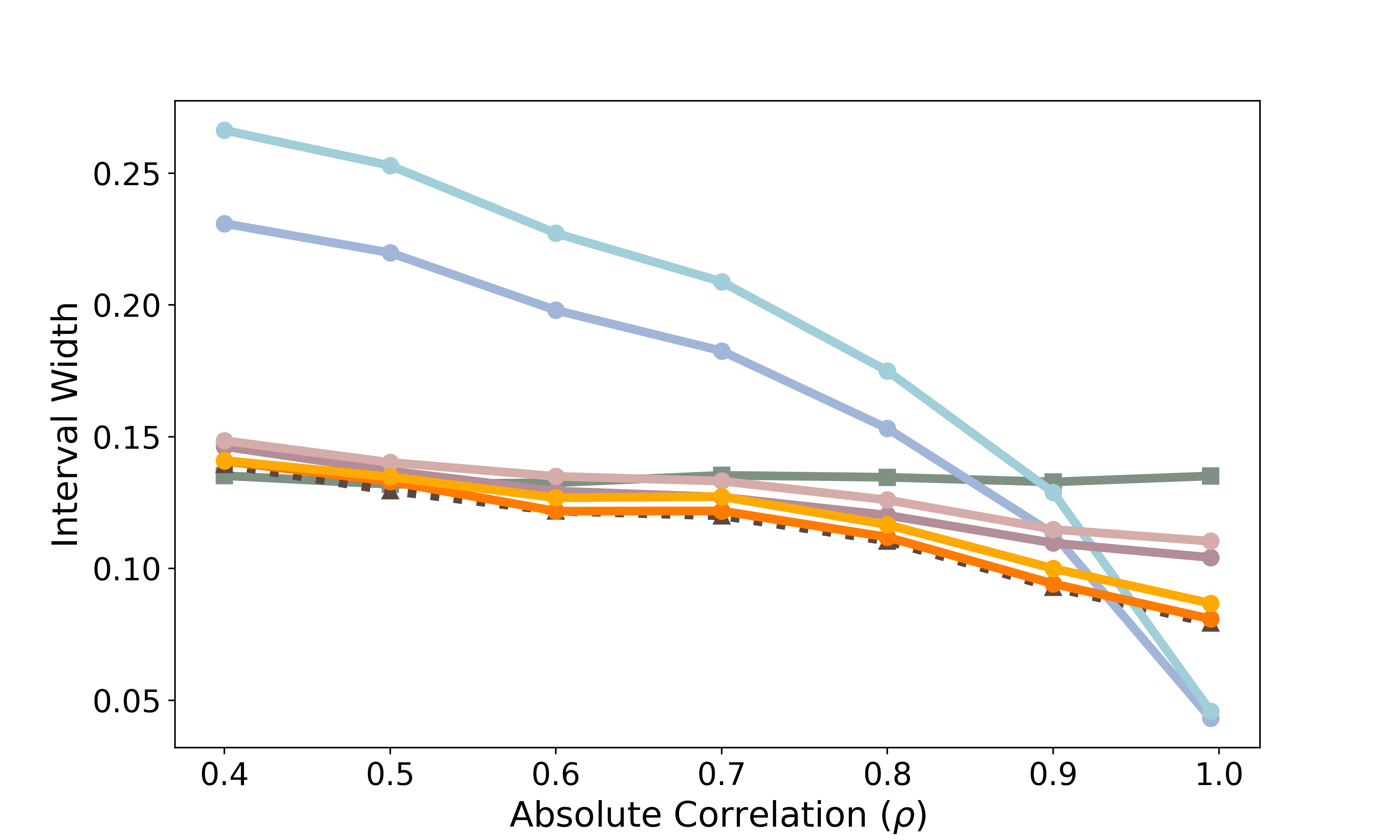}
        \caption{Interval width vs ($\rho$); $n=100$, $N=5000$, $\alpha = 0.05$}
        \label{fig:width_vs_rho_alpha_0.05}
    \end{subfigure} \\
    \begin{subfigure}[b]{0.48\textwidth}
        \centering
        \includegraphics[width=\textwidth]{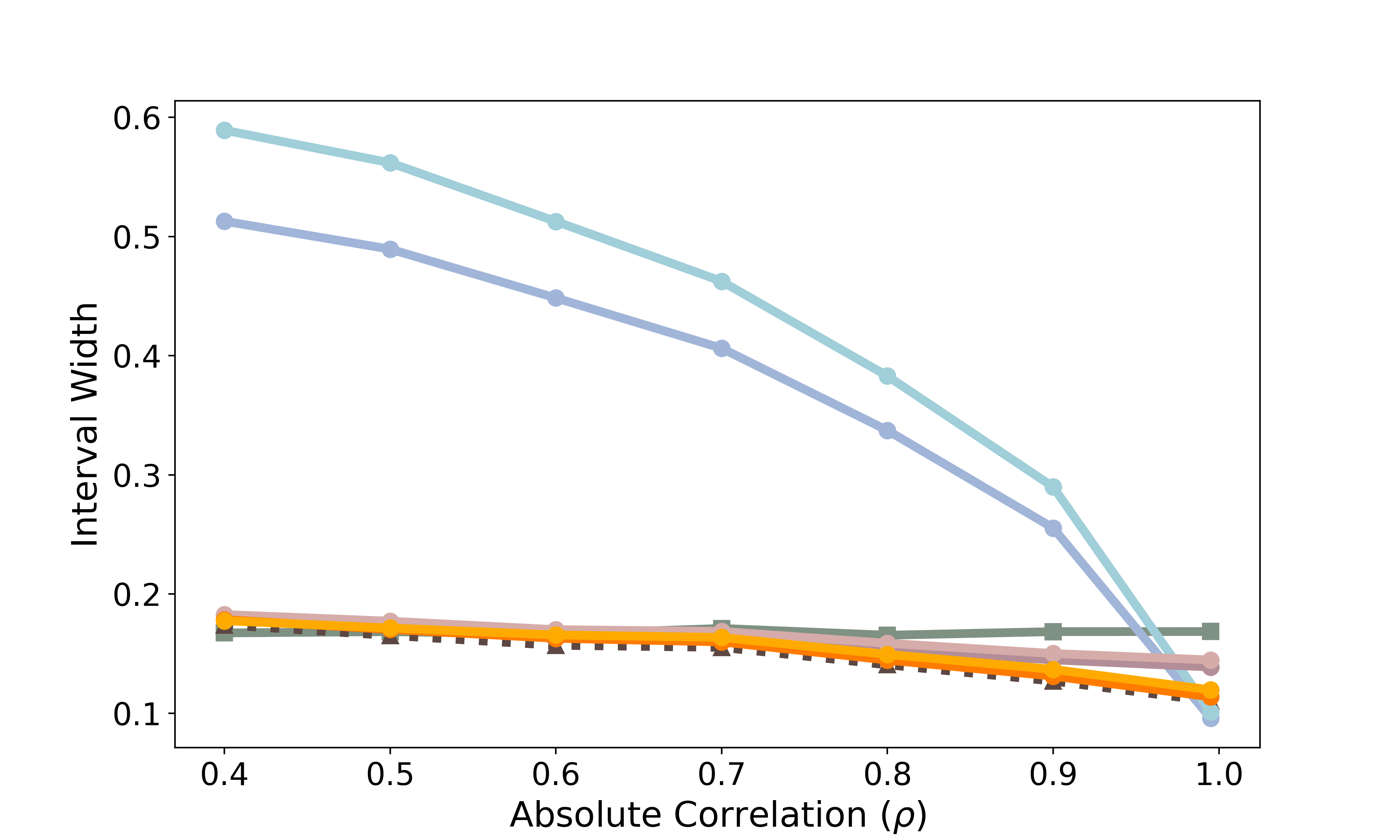}
        \caption{Interval width vs ($\rho$); $n=100$, $N=5000$, $\alpha = 0.01$}
        \label{fig:width_vs_rho_alpha_0.01}
    \end{subfigure}
    \hfill
    \begin{subfigure}[b]{0.48\textwidth}
        \centering
        \includegraphics[width=\textwidth]{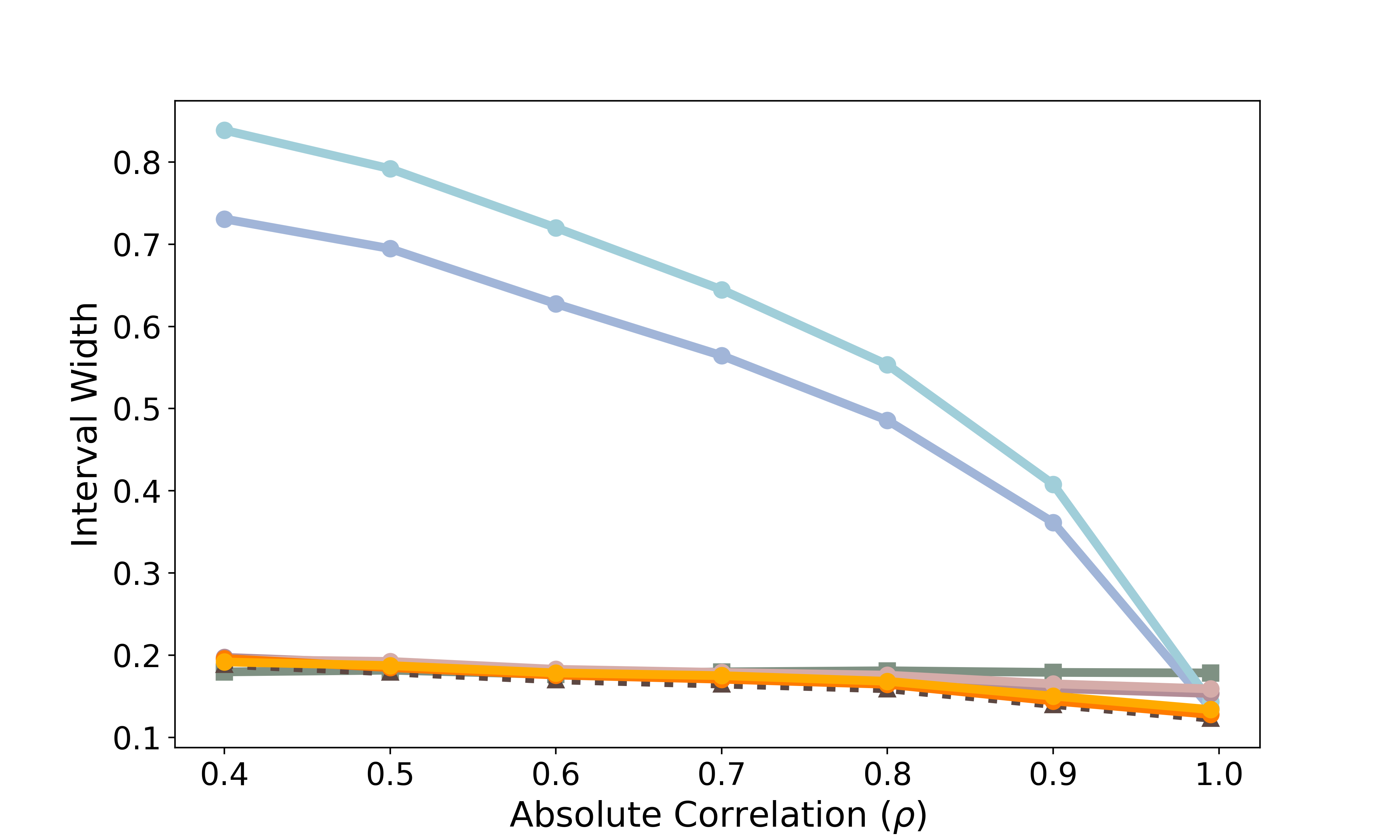}
        \caption{Interval width vs ($\rho$); $n=100$, $N=5000$, $\alpha = 0.005$}
        \label{fig:width_vs_rho_alpha_0.005}
    \end{subfigure}
    \caption{Interval widths versus true data correlation between real and paired data. As correlation increases, intervals narrow due to the greater amount of signal present. Results averaged over 100 independent redraws of data. As can be seen, \methodControlVariate has greater sensitivity to the true correlation because of the direct correspondence of $\rho$ to the attainable rectifier variance. However, the construction comes at a significant cost in lower-correlation regimes, and again illustrates strong sensitivity to $\alpha$. Furthermore, as noted in \Cref{fig:coverage_vs_nsim}, cases of very high correlation often result in miscoverage using the standard control variates implementation with Chebyshev's Inequality \cite{luo2025leveraging}.}
    \label{fig:width_vs_rho}
\end{figure}

\textbf{Coverage as Simulator Data  Grows}
\label{app:artificial_data:width_vs_nsim}
We now investigate the second thematic point, on the limitations of variance minimization as a certification for estimation efficiency. The first key comment pertains to the resulting interval coverage, to which we have access by virtue of constructing the artificial data and knowing its key features (including the mean). 

Variance minimization is generally synonymous with improving the estimator efficiency -- i.e., shrinking the interval width -- \emph{but only so long as the intervals enforce validity}. As shown in \Cref{fig:coverage_vs_nsim}, this is a challenge for \methodControlVariate, because the technique is \emph{not unbiased}. Thus, it is susceptible to excessive optimism (``trusting the simulator too much"), which leads to miscoverage when the amount of simulator data grows. As shown, this effect is most pronounced precisely when \methodControlVariate is relatively strongest (at larger $\alpha$). Importantly, for practical problems, the evaluator cannot know whether they are in an excessively optimistic regime; this is precisely the reason for enforcing \Cref{eq:guarantee} as a property of the evaluation procedure. As shown, such methods cover uniformly, while being generally efficient (recovering the rectifier variance) across different levels of $\alpha$ and $\rho$. 
\begin{figure}[h]
    \centering
    \begin{subfigure}[b]{0.48\textwidth}
        \centering
        \includegraphics[width=\textwidth]{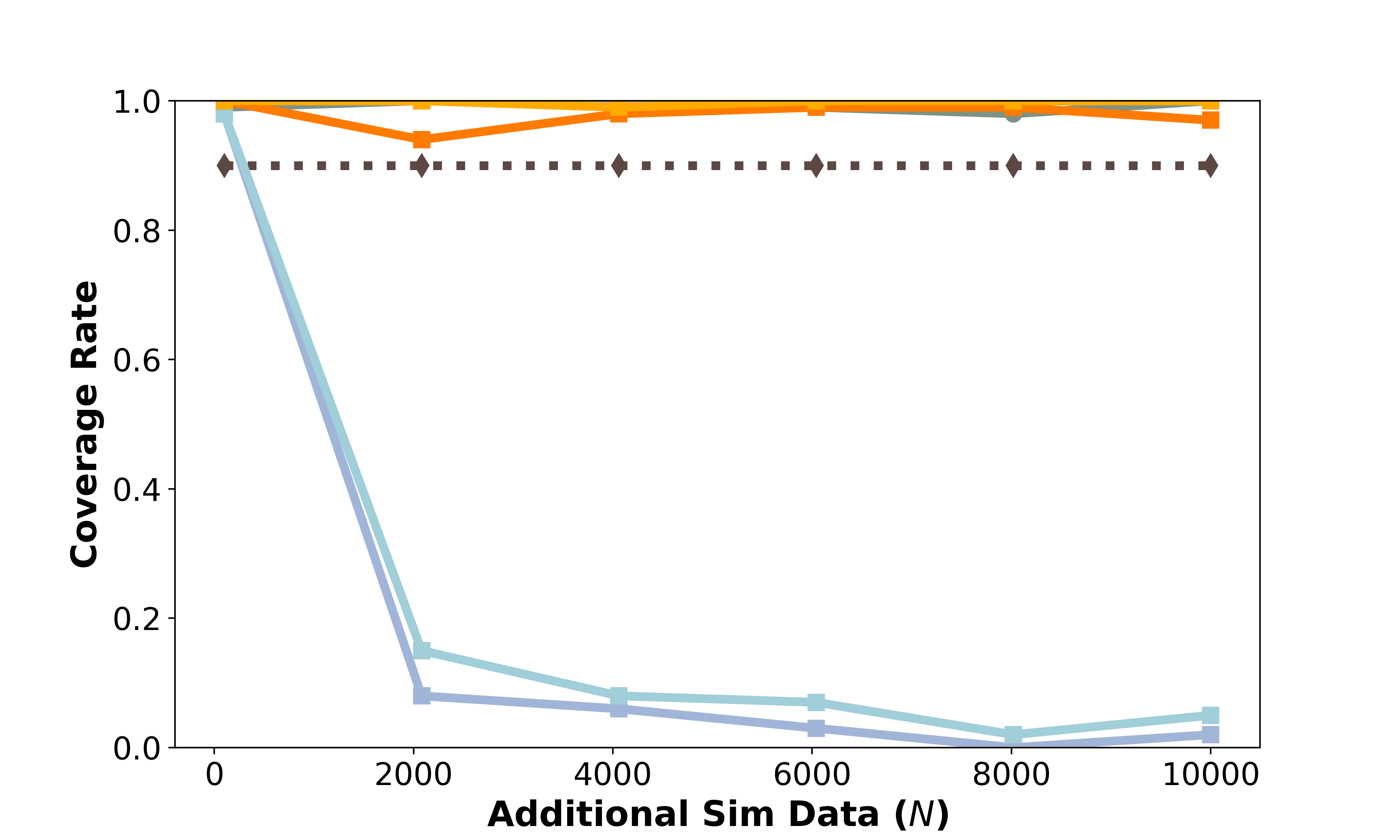}
        \caption{Coverage Rate vs ($N$): $\alpha = 0.1$, $\rho=0.97$}
        \label{fig:coverage_vs_nsim_alpha_0.1}
    \end{subfigure}
    \hfill
    \begin{subfigure}[b]{0.48\textwidth}
        \centering
        \includegraphics[width=\textwidth]{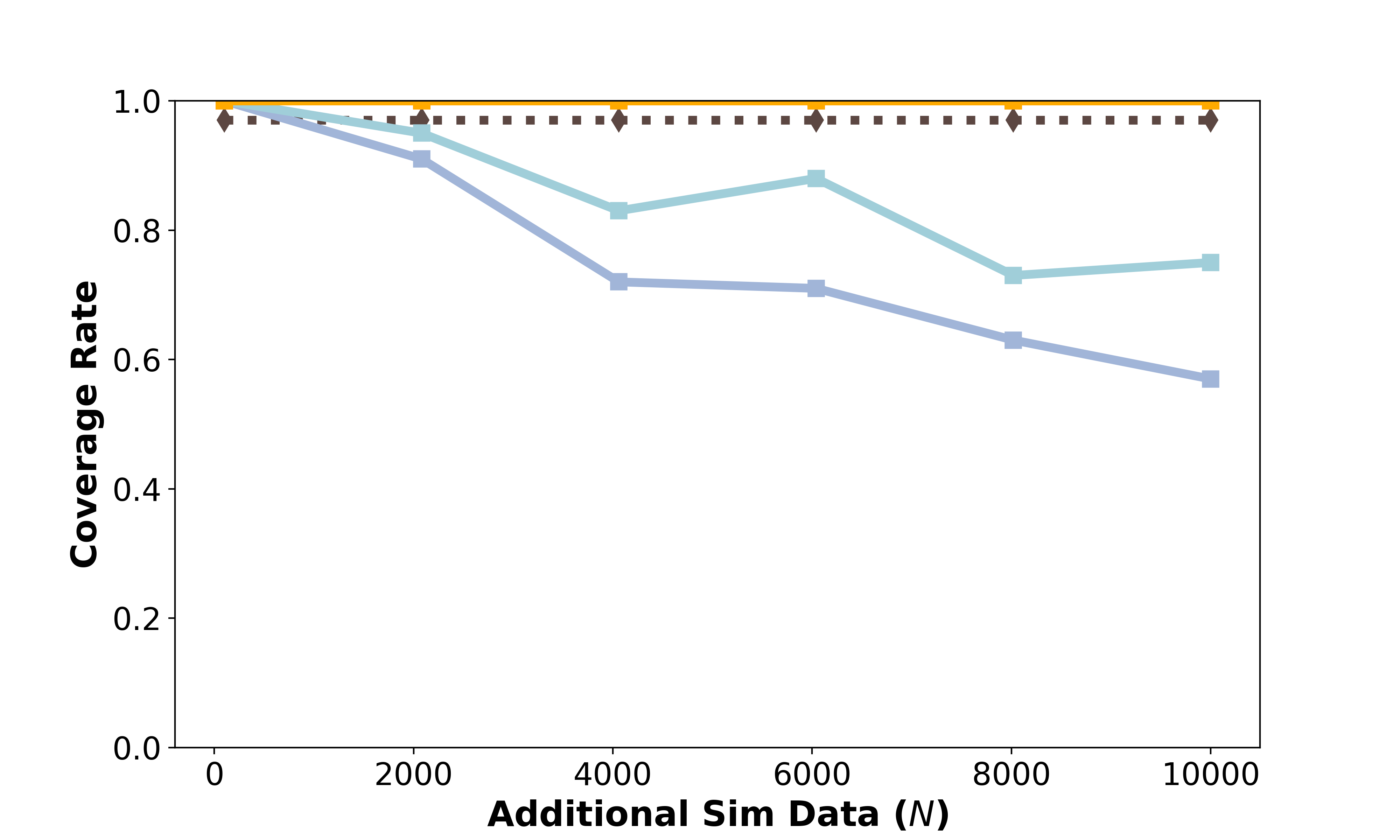}
        \caption{Coverage Rate vs ($N$): $\alpha = 0.03$, $\rho=0.97$}
        \label{fig:coverage_vs_nsim_alpha_0.03}
    \end{subfigure} \\
    \begin{subfigure}[b]{0.48\textwidth}
        \centering
        \includegraphics[width=\textwidth]{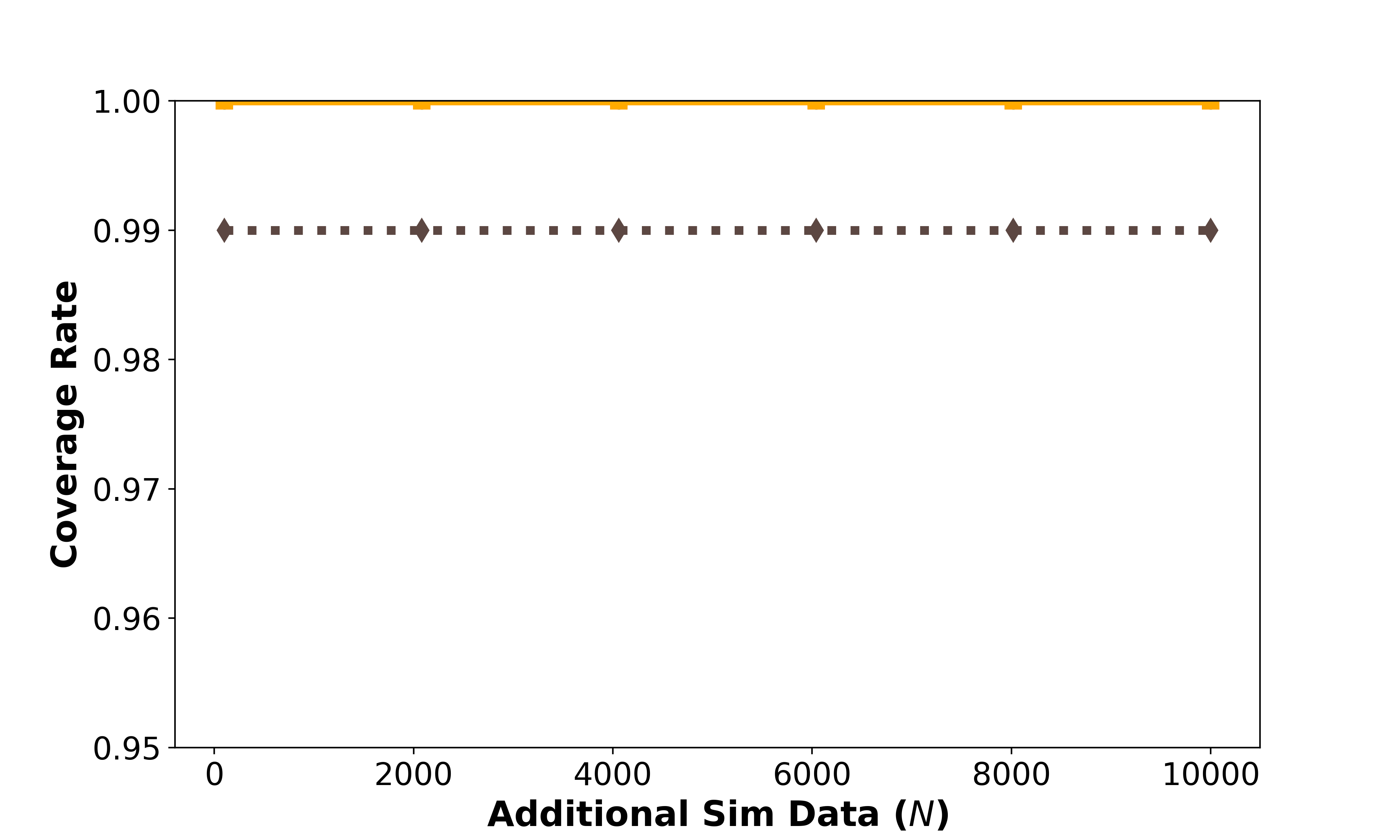}
        \caption{Coverage Rate vs ($N$): $\alpha = 0.01$, $\rho=0.97$}
        \label{fig:coverage_vs_nsim_alpha_0.01}
    \end{subfigure}
    \hfill
    \begin{subfigure}[b]{0.48\textwidth}
        \centering
        \includegraphics[width=\textwidth]{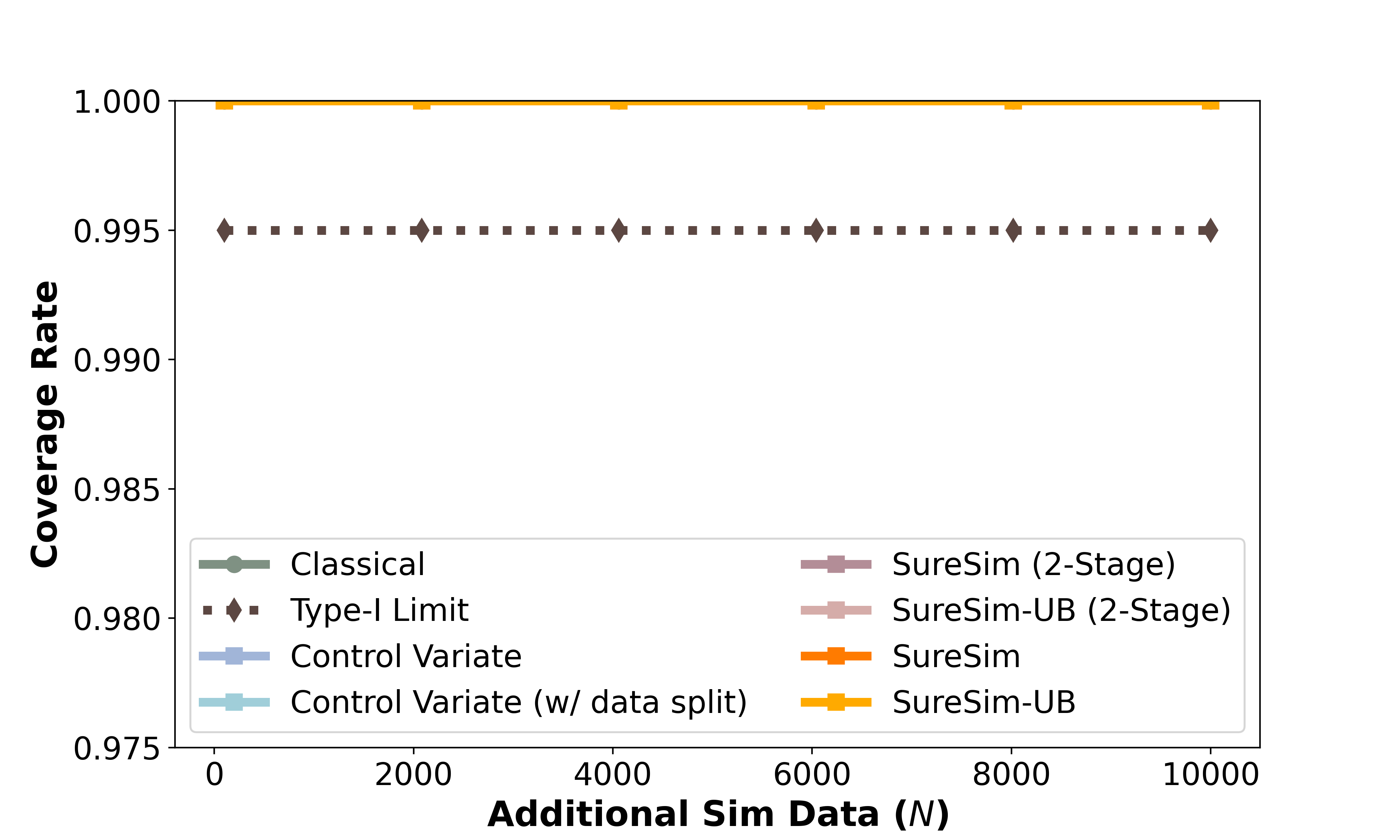}
        \caption{Coverage Rate vs ($N$): $\alpha = 0.005$, $\rho=0.97$}
        \label{fig:coverage_vs_nsim_alpha_0.004}
    \end{subfigure}
    \caption{Interval coverage rate on all methods for artificial data. Each plot shows the coverage rate against varying $N$ ($n = 100$). Results averaged over 100 independent redraws of data. The desired confidence level increases ($\alpha$ decreases) left-to-right, top-to-bottom. First, every provably nonasymptotically valid method covers in every regime, as expected. By contrast, note that \methodControlVariate fails to cover in the top row as the amount of simulation data grows; this is a result of bias in the estimator causing inconsistency. Importantly: it is \emph{precisely when \methodControlVariate intervals become narrower than our methods that they lose validity}. Thus, the only valid instance of  empirical improvement of \methodControlVariate over our procedure on this data is in the case $\alpha = 0.01$; crucially, however, the practitioner cannot know for their problem whether they are in this regime (as for different problem instances, the critical value of $\alpha$ will differ in general from $0.01$). Thus, the practitioner may be in the (narrow) valid and efficient regime, or may be in the invalid (too-optimistic) regime, or in the inefficient regime -- and will not be able to distinguish which one. }
    \label{fig:coverage_vs_nsim}
\end{figure}
\textbf{Interval Width versus Significance Level}
\label{app:artificial_data:significance}
Finally, we evaluate the effect of changes in significance level on the interval scaling. As the \methodPPIUnif and \methodPPINonAsym family of procedures apply WSR\cite{waudby2024estimating}, which admits bounded (and therefore, sub-Gaussian) random variables, these methods can recover $\mathcal{O}(\log{\frac{1}{\alpha}})$ scaling as $\alpha \rightarrow 0^+$. This is broadly indicative of Hoeffding- or Bernstein-type concentration bounds. By contrast, \methodControlVariate is slightly more general (not requiring boundedness), but loses this scaling rate as a consequence. This explains the significantly increased sensitivity to $\alpha$ (specifically, $\mathcal{O}(\frac{1}{\sqrt{\alpha}})$ scaling) obtained via Chebyshev's inequality. 

The tradeoffs of this design choice can be seen in several intuitive ways. First, naturally-unbounded metrics (e.g., log likelihoods) are more suited to \methodControlVariate. However, from a practical standpoint, guarantees requiring very high confidence (often the best that statistical assurances can achieve for safety-critical applications) will scale much more efficiently with our methods; that is, \methodControlVariate will very often yield vacuous intervals for, e.g., $\alpha < 0.0001$, especially when $n$ is constrained. The method can tighten this by instead using a Hoeffding-type bound, but then loses the generality that sets it apart from WSR-based procedures, as it must also enforce a boundedness constraint.

\subsection{Key Takeaways}
\label{artificial_data:conclusion}
A recurring theme of the methodological comparison given in the preceding sections is the relative strength of \methodControlVariate procedures for (a) less stringent significance requirements (higher $\alpha$), and (b) higher correlation $\rho$. Intuitively, it is better able to exploit `easier' settings to tighten intervals more aggressively, at the cost of underperforming (in terms of downstream interval width) in `harder' ones, where greater confidence is required and the amount of signal in the proxy data is limited. However, in small \(n\) settings, the optimistic choice of the control variate coefficient can lead to miscoverage. 
\section{Implementation Details}
\label{appendix:implementation_details}

\textbf{Hardware Setup}
We use a Franka Panda robot for our real robot experiment. For all experiments, we use joint space control at 15Hz. We use Logitech C920 webcam as our third person camera, and RealSense D405 for the wrist camera. Both cameras use resolution $192\times192$. We use a Meta Quest 2 VR headset for teleoperation to perform data collection.


\subsection{Policy Implementation Details}
We implement Diffusion Policy~\cite{chi2023diffusion} and $\pi_0$ \cite{black2024pi_0} for our two experiments respectively. Below we detail their implementations.

\textbf{Diffusion Policy.} We follow the original implementations from \citet{chi2023diffusion}, and use ResNet-18~\cite{he2015deepresiduallearningimage} as our vision encoder. The policy takes in two images from the wrist camera and the third person camera, as well as the robot state. The robot state is an 8‑dimensional vector comprising the seven joint positions and a single gripper binary state. Action output is specified as 8‑dimensional target absolute joint‑position and a target gripper state.
We train the policy with 50000 gradient updates with a fixed batch size of 64. The training can be finished in approximately one wall-clock hour on a Nvidia L40 GPU. We augment the input images with standard color‑jitter and random rotation during training. A complete list of hyper‑parameters is provided in Table \ref{table:app_hyperparam_sim_dp}.

\begin{table*}[h]
\centering
\small
\renewcommand{\arraystretch}{1.05}
\caption{Hyper-parameters of simulation diffusion policy.}
\begin{tabular}{
  c
  c
  c
  c
  c
  c
  c
}
\toprule
 Model Dimension & Dim Mults & Time Embedding Dimension & History Steps & Horizon & Action Steps \\
\midrule
\addlinespace
128 & [1,2,4]  & 128  & 1 & 16 & 8 \\
\bottomrule
\end{tabular}
\label{table:app_hyperparam_sim_dp}
\end{table*}

\textbf{Multi-task $\pi_0$.} 
We fine-tune from $\pi_0$-base model. Similar to Diffusion Policy, $\pi_0$ takes in two images from the wrist camera and the third person camera, proprioception states, and additionally a task language instruction. 
\begin{wrapfigure}{r}{0.3\textwidth}
\centering
    \includegraphics[width=0.3\textwidth]{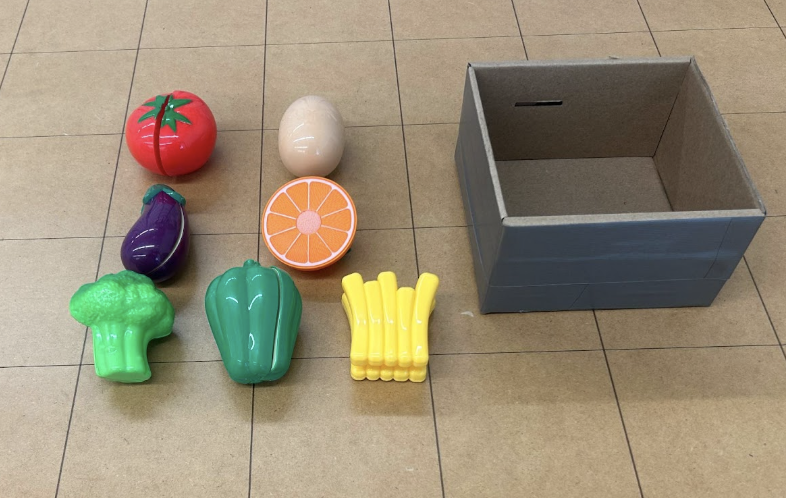}
\caption{Objects used to train multi-object policy.}
\label{fig:training_objects}
\end{wrapfigure}
The instruction is formatted as ``Put $<$object$>$ into the box". We freeze the VLM and only train the action expert. We train all policies for 15,000 gradient steps with batch size 64. The training can be finished in 8 wall-clock hours, parallelized on 4 Nvidia L40 GPUs. For other hyper-parameters, we follow the default setting from \cite{black2024pi_0}. Initially, we had trained diffusion policy for putting tomato on the plate. While fine-tuning \(\pi_0\), we chose to replace the plate with a box to make the task semantically more meaningful with a variety of objects.

\subsection{Paired and Additional Simulation Data}
A key assumption in our work is that the environments for the paired evaluation and the additional simulation evaluations are drawn i.i.d from the same distribution. We assume that the objects used for real evaluation are representative of and randomly sampled from some universal distribution of kitchen objects. As discussed in the limitations section,  operationalizing this pipeline would require a community-wide investment in building large datasets for evaluation. Despite this limitation, the following histograms in~\Cref{fig:histograms} illustrate qualitative similarity in the paired and additional simulation evaluations.

\begin{figure}[h]
\centering
\begin{subfigure}[b]{0.48\textwidth}  
        \centering
        \includegraphics[width=\textwidth]{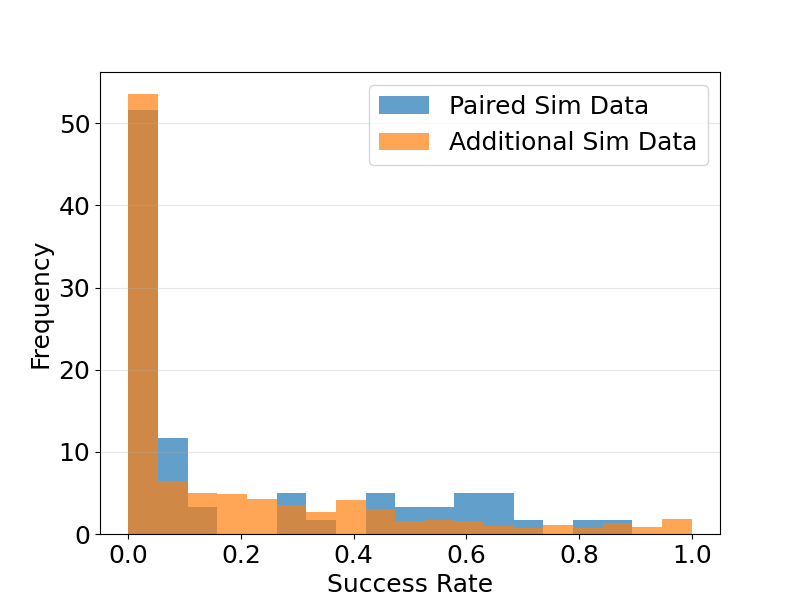}
        \caption{}
        \label{fig:franka_hist}
    \end{subfigure}
    \hfill
    \begin{subfigure}[b]{0.48\textwidth}
        \centering
        \includegraphics[width=\textwidth]{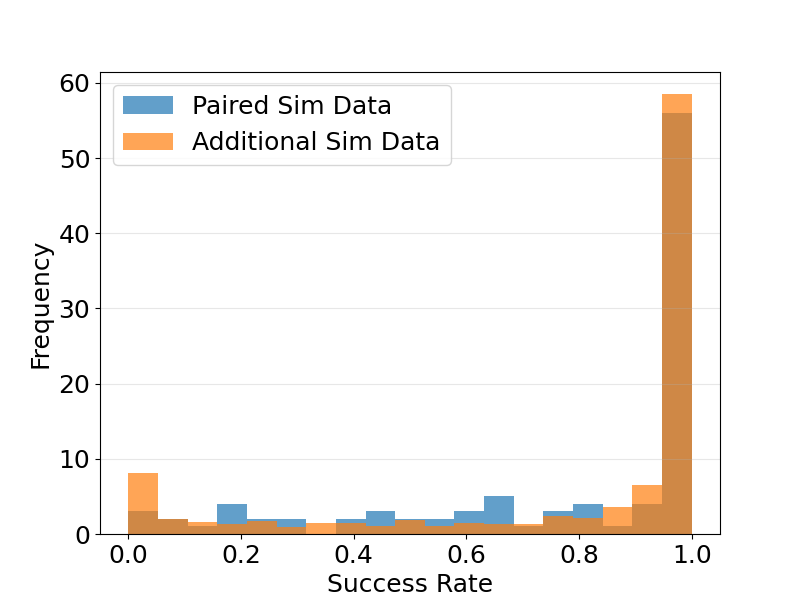}
        \caption{\(\pi_0\)}
        \label{fig:pi_0_hist}
    \end{subfigure}
\caption{Simulation success scores of policies on 3D object models used for paired and additional simulations}
\label{fig:histograms}
\end{figure}
}


\end{document}